\documentclass{article} 

\usepackage[table]{xcolor}
\usepackage{preprint,times}

\renewcommand{\sc}{\bfseries}

\usepackage{amsmath,amsfonts,bm}

\def\eqref#1{Equation~\ref{#1}}

\def\1{\bm{1}}

\DeclareMathAlphabet{\mathsfit}{\encodingdefault}{\sfdefault}{m}{sl}
\SetMathAlphabet{\mathsfit}{bold}{\encodingdefault}{\sfdefault}{bx}{n}

\usepackage{hyperref}
\usepackage{url}
\usepackage{subcaption}
\usepackage{booktabs}
\usepackage{multirow}
\usepackage{graphicx}
\usepackage{amsfonts, amsmath, amsthm, amssymb} 
\usepackage{bm}
\usepackage{bbm}
\usepackage{marvosym}
\usepackage{mathtools}
\usepackage[ruled,vlined,linesnumbered]{algorithm2e}
\usepackage{natbib}
\usepackage[nameinlink]{cleveref}
\usepackage{adjustbox}   
\usepackage{enumitem}   
\usepackage{hyperref}
\usepackage{tikz}
\usepackage[most]{tcolorbox}
\usepackage{tabularx}
\usepackage{array}

\tcbset{casebox/.style={%
  enhanced, breakable,
  arc=3pt, boxrule=1pt, titlerule=0.5pt,
  fonttitle=\bfseries,
  left=7pt, right=7pt, top=5pt, bottom=5pt,
  toptitle=3pt, bottomtitle=3pt}}
\newtcolorbox{problemcase}[1][]{casebox,
  colframe=black!60, colback=black!4, colbacktitle=black!8,
  coltitle=black!60,
  title={Problem\hfill{\normalfont\footnotesize #1}}}
\newtcolorbox{opdcase}[1][]{casebox,
  colframe=red!55!black, colback=red!55!black!4!white,
  colbacktitle=red!55!black!8!white, coltitle=red!55!black,
  fontupper=\small\ttfamily,
  title={OPD Student\hfill{\normalfont\footnotesize #1}}}
\newtcolorbox{tidecase}[1][]{casebox,
  colframe=green!40!black, colback=green!40!black!4!white,
  colbacktitle=green!40!black!8!white, coltitle=green!40!black,
  fontupper=\small\ttfamily,
  title={\method{} Student\hfill{\normalfont\footnotesize #1}}}

\newcommand{\method}{TIDE}
\definecolor{citecolor}{HTML}{2A6F8F}
\definecolor{figrefcolor}{HTML}{A23B5A}
\hypersetup{
    colorlinks=true,
    citecolor=citecolor,
    linkcolor=figrefcolor,
    urlcolor=citecolor
}

\definecolor{tideblue}{HTML}{08519C}
\newtcolorbox{takeaway}[1]{%
  enhanced, breakable,
  arc=3pt, boxrule=1pt, titlerule=0.5pt,
  colframe=tideblue, colback=tideblue!4!white,
  colbacktitle=tideblue!8!white, coltitle=tideblue,
  fonttitle=\bfseries,
  left=7pt, right=7pt, top=5pt, bottom=5pt,
  toptitle=3pt, bottomtitle=3pt,
  title={#1}}
  
\crefname{figure}{Fig.}{Figs.}
\Crefname{figure}{Fig.}{Figs.}
\crefname{table}{Table}{Tables}
\Crefname{table}{Table}{Tables}
\crefname{section}{Section}{Sections}
\Crefname{section}{Section}{Sections}
\crefname{subsection}{Section}{Sections}
\Crefname{subsection}{Section}{Sections}
\crefname{appendix}{Appendix}{Appendices}
\Crefname{appendix}{Appendix}{Appendices}
\crefname{subappendix}{Appendix}{Appendices}
\Crefname{subappendix}{Appendix}{Appendices}
\crefname{equation}{Equation}{Equations}
\Crefname{equation}{Equation}{Equations}
\crefname{algorithm}{Algorithm}{Algorithms}
\Crefname{algorithm}{Algorithm}{Algorithms}

\usetikzlibrary{arrows.meta,positioning,calc}
\definecolor{groupbg}{HTML}{ECEFF3}   
\definecolor{oursbg}{HTML}{E6F0FA}    
\definecolor{gaincol}{HTML}{0B7A3B}   
\definecolor{dropcol}{HTML}{B3261E}   

\newtheorem{proposition}{Proposition}
\DeclareMathOperator{\sg}{sg}
\newcommand{\gain}[1]{\textcolor{gaincol}{#1}}
\newcommand{\drop}[1]{\textcolor{dropcol}{#1}}

\crefname{theorem}{Theorem}{Theorems}
\Crefname{theorem}{Theorem}{Theorems}
\crefname{proposition}{Proposition}{Propositions}
\Crefname{proposition}{Proposition}{Propositions}
\crefname{assumption}{Assumption}{Assumptions}
\Crefname{assumption}{Assumption}{Assumptions}
\crefname{definition}{Definition}{Definitions}
\Crefname{definition}{Definition}{Definitions}
\crefname{corollary}{Corollary}{Corollaries}
\Crefname{corollary}{Corollary}{Corollaries}

\makeatletter
\let\iclr@originalfnsymbol\fnsymbol
\renewcommand{\fnsymbol}[1]{%
  \ifnum\value{#1}=1\relax
    \text{\Letter}%
  \else
    \iclr@originalfnsymbol{#1}%
  \fi}
\makeatother

\title{Mismatch Matters: On-Policy Distillation Beyond Token Agreement}
\author{
\textbf{Zichao Yu}$^{1}$ \quad
\textbf{Chengzhi Yu}$^{2}$ \quad
\textbf{Shengze Xu}$^{3}$ \quad
\textbf{Yujin Han}$^{1}$ \\
\textbf{Bingqing Jiang}$^{1}$ \quad
\textbf{Xu Wang}$^{1}$ \quad
\textbf{Difan Zou}$^{1}$\thanks{Corresponding author: dzou@cs.hku.hk} \\[5pt]
{\normalfont\small
$^{1}$The University of Hong Kong \quad
$^{2}$University of Science and Technology of China} \\
{\normalfont\small
$^{3}$The Chinese University of Hong Kong}
}

\iclrpreprintcopy
\begin{document}

\maketitle

\begin{abstract}
On-policy distillation (OPD) has emerged as a core component of modern LLM post-training pipelines, yet we reveal a failure mode: \emph{degenerate agreement}, where students exploit repetitive loops to achieve near-perfect token agreement with the teacher despite globally flawed responses \footnote{Representative examples of degenerate agreement are provided in
\cref{app:case_studies}.}.
We therefore shift our focus from agreement to teacher--student mismatch and find that the mismatch tokens can be mainly categorized to two types: student-excess tokens and student-deficit tokens. Specifically, student‑excess tokens are those generated by the student but assigned near‑zero probability by the teacher; their log‑ratio corrections grow unbounded and destabilize the update. Student‑deficit tokens, in contrast, are preferred by the teacher but rarely sampled by the student; their absence blocks the transfer of the teacher’s reasoning patterns. 
To tackle these mismatch directions, we propose \textbf{TIDE} (Token-level
Independent Deficit--Excess correction), which applies bounded Hellinger shaping to suppress the most severe sampled excesses and an
analytic teacher top-\(K\) injection to restore deficient probability mass without requiring deficit tokens to be sampled. Across mathematical reasoning benchmarks with multiple Qwen3 teacher--student pairs, TIDE consistently outperforms standard OPD and recent token-selection and reward-shaping baselines. Besides, the gains of TIDE are more pronounced
under strong teacher--student mismatch, where it improves Avg@8 from \(6.9\%\) to \(20.3\%\), reduces average response length by a factor of
\(3.6\), and substantially reduces formatting failures. 
Our code is available at
\url{https://github.com/yzc-666/TIDE}.

\end{abstract}
\begin{figure}[h]
\centering
\includegraphics[width=\textwidth]{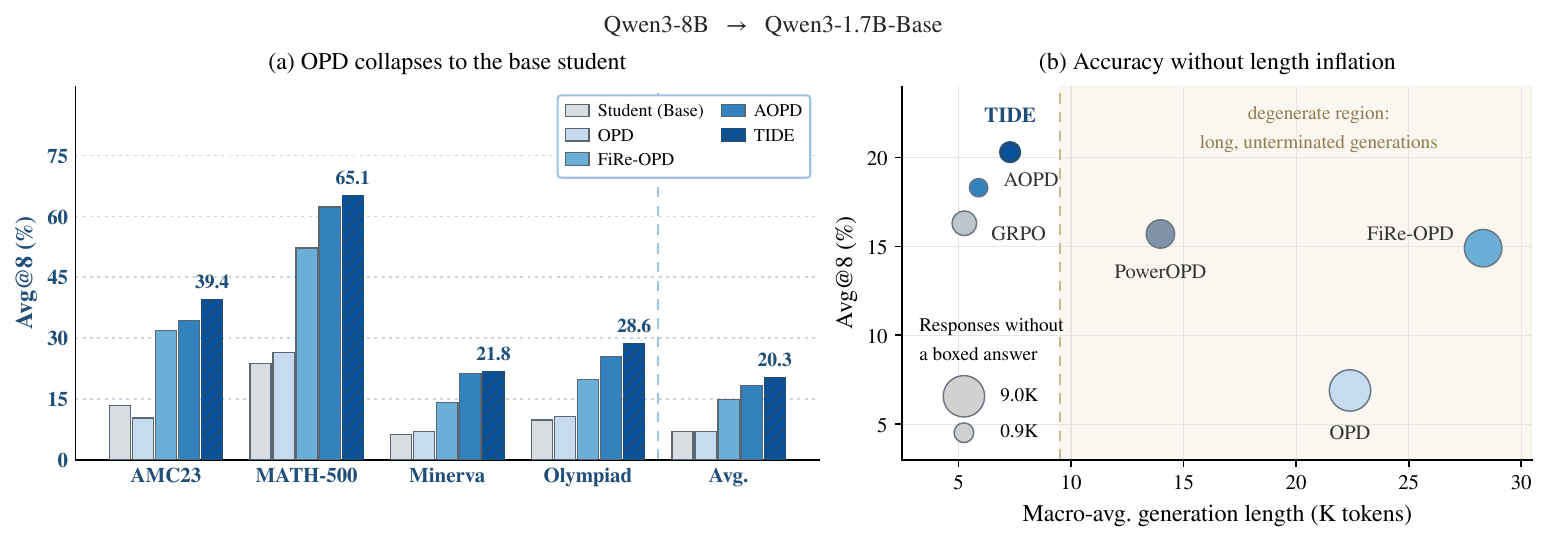}
\caption{\textbf{\method{} achieves the highest accuracy without length
inflation.} \textbf{(a)} \method{} consistently outperforms the baselines in
Avg@8. \textbf{(b)} Unlike OPD and FiRe-OPD, \method{} maintains short,
well-formed generations.}
\label{fig:teaser}
\end{figure}
\section{Introduction}

On-policy distillation (OPD)~\citep{lu2025onpolicydistillation} has emerged as a powerful and widely adopted paradigm for post-training large language models (LLMs)~\citep{agarwal2024policy,gu2024minillm}, and has become a standard component in recent open-weight frontier models such as Qwen3, MiMo-V2-Flash, GLM-5, and Kimi K3~\citep{yang2025qwen3,xiao2026mimo,zeng2026glm,team2026kimi}. Rather than relying on fixed expert demonstrations, OPD samples trajectories from the student policy and queries an external teacher at the states the student actually visits. This on-policy supervision mitigates the exposure bias of supervised fine-tuning (SFT)~\citep{bengio2015scheduled,ross2011reduction,agarwal2024policy}, while offering denser token-level feedback than the sparse outcome rewards used in reinforcement learning with verifiable rewards (RLVR)~\citep{shao2024deepseekmath,guo2025deepseek,yu2026dapo}.

\begin{figure}[t]
    \centering
    \includegraphics[width=0.95\columnwidth]
    {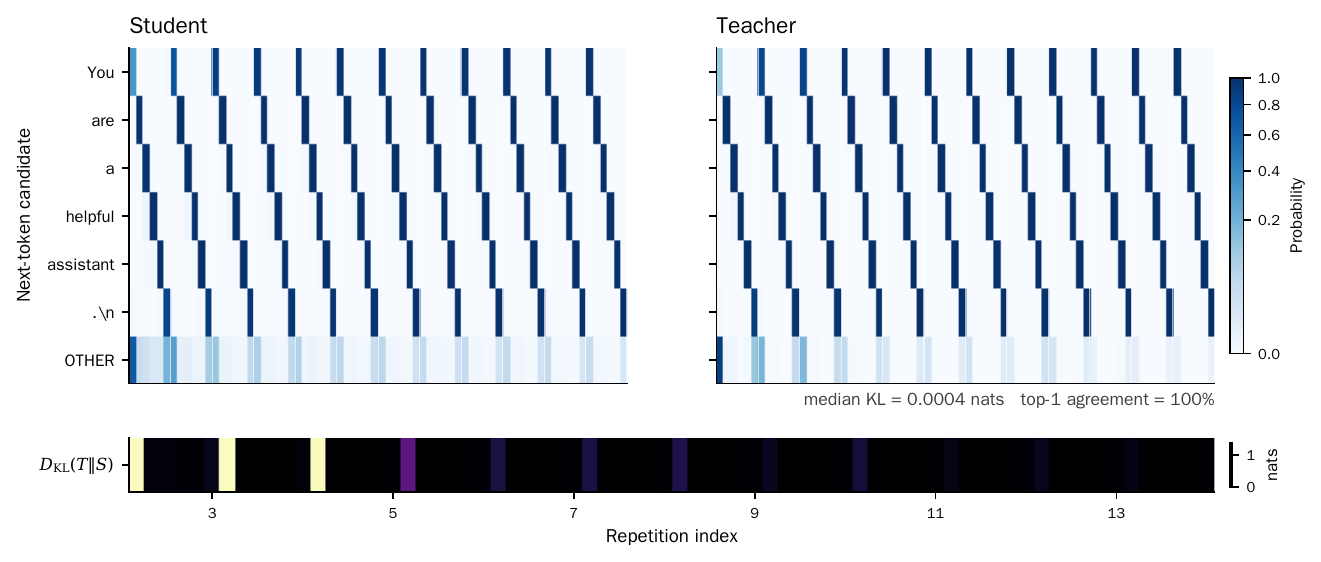}
    \caption{
    \textbf{Repetition creates degenerate agreement.}
    Visualization of a representative OPD rollout in which the student
    repeatedly generates the sequence
    \texttt{You are a helpful assistant.\textbackslash n}. The left and right
    heatmaps show the student and teacher next-token distributions conditioned
    on progressively longer repetitive prefixes; rows correspond to tokens in
    the repeating unit, and darker colors indicate higher probability. As
    repetition continues, both models concentrate on the same periodic
    continuation: their top-1 predictions agree at \(100\%\), while the median
    teacher--student KL falls to \(0.0004\) nats, as shown in the bottom strip.
    The globally degenerate rollout therefore appears locally well-aligned
    under token-level supervision.
    }
    \label{fig:repetition-observation}
\end{figure}

Despite these advantages, OPD is prone to abrupt length inflation, with
student generations drifting toward the decoding cap~\citep{luo2026demystifying}.
Our diagnostics trace this to persistent repetitive loops (see detailed analysis in \cref{sec:teacher-hacking} and  one example in~\Cref{fig:repetition-observation}): the student repeatedly generates an identical prefix
\texttt{"You are a helpful assistant.\textbackslash n"};
conditioned on the student-generated prefix, the teacher finds each additional repetition increasingly predictable, until the two distributions (generated by student and teacher) nearly coincide (median KL \(0.0004\) nats).  Critically, the student can thus drive teacher–student divergence to near zero while the overall response remains globally degenerate, revealing that high token-level agreement can actively mask global degeneration. We term this phenomenon \emph{degenerate agreement}, and the underlying mechanism \emph{student-induced teacher hacking}: repetition renders a degenerate continuation locally consistent to the teacher. This also resembles the low‑KL trap recently observed for corrupted student prefixes~\citep{xin2026escaping}.

A prevalent principle in recent OPD research is that training should focus on high teacher–student overlap to obtain valuable supervision. This assumption underpins recent work on token selection~\citep{li2026rethinking,zhao2026prefix,yang2026prune}, which treats high-overlap tokens as the most useful learning signals. However, this overlap-based criterion may also suffer from the risk of degenerate agreement: repetitive generations can produce artificially high overlap and then the teacher signal offers little new information. In such cases, high agreement ceases to be a reliable indicator of supervisory quality. This motivates us to shift the focus from agreement to mismatch (or disagreement) as the basis for supervision, and explore divergent tokens as an alternative, potentially more informative source of learning signal.

Concretely, we find that two types of mismatch tokens are rather distinct and important (see Section \ref{sec:prelim-opd} for details): student-excess and student-deficit tokens. In particular, \emph{Student-excess tokens} are overproduced by the student relative to the teacher; they are readily observed in rollouts, yet their log-ratio difference, can grow unbounded as the teacher probability approaches zero, destabilizing updates. 
For instance, in our OPD training, the most negative \(1\%\) of tokens
contribute nearly half of the total gradient.
\emph{Student-deficit tokens}, conversely, are favored by the teacher but underweighted by the student. These tokens have a median probability of just \(2.04\times10^{-4}\); under our four-rollout budget, \(95.5\%\) of them are sampled with less than \(1\%\) probability. Despite this inaccessibility, we find that these deficit tokens carry rather valuable supervisory signal missed by standard on-policy supervision (we perform a simple deficit-only intervention and raise Avg@8 from \(6.87\) to \(18.32\) (see~\cref{app:deficit-accessibility}). 

\begin{figure*}[t]
\centering
\includegraphics[width=\textwidth]{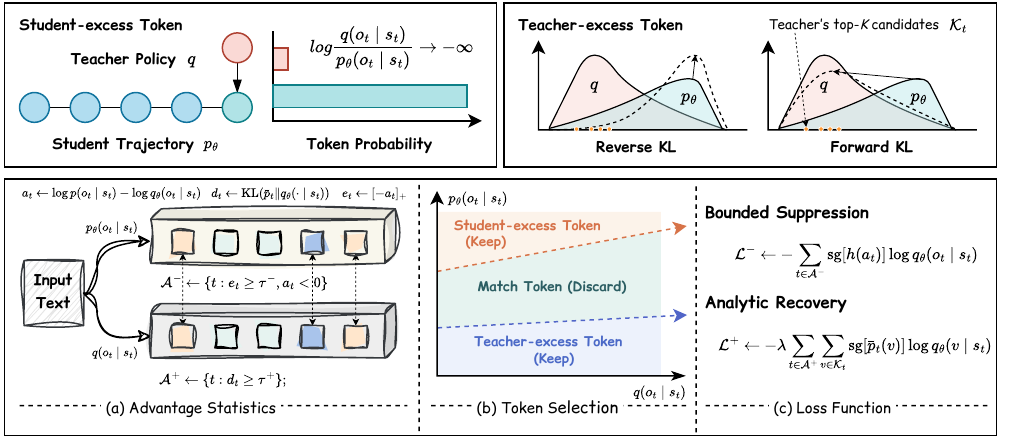}
\caption{
\textbf{An overview of TIDE pipeline.}
Student-excess tokens are readily observed in student rollouts, but can produce unbounded negative advantages when assigned near-zero probability by the teacher. In contrast, student-deficit tokens are unlikely to be sampled and are therefore identified analytically using forward KL over the teacher's top-\(K\) candidates. TIDE computes branch-specific mismatch statistics (a), independently retains positions with severe excess or deficit while discarding low-mismatch positions (b), and applies bounded suppression to excess tokens and analytic probability-mass recovery to deficit tokens (c).
}
\label{fig:tide-pipeline}
\end{figure*}


Based on our findings on the token-level mismatch, we propose \textbf{TIDE}
(\textbf{T}oken-level \textbf{I}ndependent
\textbf{D}eficit--\textbf{E}xcess correction), a correction mechanism that deliberately leverages mismatch tokens to achieve more informative and stable on-policy distillation (OPD) training. For student
excess tokens, TIDE replaces the unbounded log‑ratio with a principled Hellinger‑shaped weight derived from the squared Hellinger divergence, which enforces bounded suppression without heuristic clipping or allowing extreme tokens to dominate the update. For
student deficit tokens, TIDE introduces an analytic teacher top‑
$K$ objective to restore missing probability mass directly, bypassing the need for rare rollouts. At each mini‑batch, TIDE independently evaluates both mismatch directions and selectively applies the corresponding correction only to the most severe positions, discarding all neutral tokens from the loss. We further develop theoretical results demonstrating that the excess branch provides bounded, locally faithful corrections aligned with a proper divergence geometry, while the deficit branch directly addresses an irreducible coverage gap that sampling‑based methods cannot resolve. Empirical evaluations on mathematical reasoning benchmarks and various teacher–student pairs (including Qwen3 family models \citep{yang2025qwen3}) confirm that TIDE consistently outperforms strong baselines, with ablation studies further validating the complementary roles of its two correction components. Our contributions are summarized as follows.


\begin{itemize}[leftmargin=*]
    \item We identify \emph{student-induced teacher hacking} as a previously
    overlooked failure mode of OPD. Once the student enters a repetitive
    prefix, the teacher can favor the same continuation, creating a low-KL
    matched region in which degeneration receives little corrective signal.

    \item We reveal an intrinsic asymmetry in informative teacher--student
    disagreement. Student-excess tokens are readily observed but can induce
    unstable updates, whereas student-deficit tokens carry valuable teacher
    information but are rarely observed in student rollouts. 

    \item We introduce TIDE, which combines quantile-gated Hellinger
    suppression for student excess with analytic teacher-top-\(K\) recovery
    for student deficit while removing redundant supervision from matched
    tokens. TIDE consistently outperforms
strong OPD baselines; notably, for our Qwen experiment setting, it raises Avg@8 from $6.87$ to $20.20$ while reducing
the average response length from $22{,}395$ to $7{,}294$ tokens.
\end{itemize}

\section{Problem Setup and Empirical Observations}
\label{sec:problem-setup}

\subsection{Problem Setup}
\label{sec:prelim-opd}

We begin by establishing the basic setup. Let $\mathcal D$ denote the distribution over input prompts $x$, and let $\mathcal V$ be the vocabulary. At a given state $s$, the fixed teacher and the trainable student define next-token distributions $p(\cdot| s)$ and $q_\theta(\cdot| s)$, respectively, where $\theta$ denotes the student parameters. Since both models use a softmax with full support over the vocabulary, these probabilities are strictly positive for every token $v\in\mathcal V$. For a prompt $x\sim\mathcal D$, the student autoregressively generates a response $o=(o_1,\ldots,o_T)\sim q_\theta(\cdot| x)$, where $T$ denotes the total response length, including the end-of-sequence token. At each position $t$, the state visited by the student is $s_t=(x,o_{<t})$. The corresponding trajectory likelihoods factorize autoregressively in the usual way: for instance, $q_\theta(o| x)=\prod_{t=1}^T q_\theta(o_t| s_t)$, and analogously for the teacher $p(o| x)$.

The objective of OPD is to minimize the expected trajectory-level reverse KL divergence:
\begin{equation}
\label{eq:opd}
\mathcal L_{\mathrm{OPD}}(\theta)
=\mathbb E_{x\sim\mathcal D}\!\left[
D_{\mathrm{KL}}\!\left(q_\theta(\cdot| x)\,\|\,p(\cdot| x)\right)
\right]
=\mathbb E_{\substack{x\sim\mathcal D\\o\sim q_\theta(\cdot| x)}}
\!\left[\sum_{t=1}^{T}\log
\frac{q_\theta(o_t| s_t)}{p(o_t| s_t)}\right],
\end{equation}
where $D_{\mathrm{KL}}$ is the Kullback--Leibler divergence.

Following recent practice~\citep{lu2025onpolicydistillation}, we implement this objective using a detached token-level policy-gradient estimator~\citep{schulman2017proximal}:
\[
g_{\mathrm{OPD}}(\theta)
=-\mathbb E\!\left[\sum_{t=1}^{T}\sg[a_t]\,
\nabla_\theta\log q_\theta(o_t| s_t)\right],
\qquad
a_t=\log\frac{p(o_t| s_t)}{q_\theta(o_t| s_t)}.
\] 
Here the expectation is taken over the same prompts and student rollouts as in \eqref{eq:opd}, and \(\sg\) denotes the stop-gradient operation. We interpret the scalar \(a_t\) as the advantage of the sampled token at step \(t\): a positive value increases its likelihood, while a negative one decreases it. A crucial subtlety is that this update rule only directly affects tokens actually drawn from the student's current distribution, i.e., a dependence that becomes especially consequential when the teacher and student predictions diverge (see more discussion in \cref{sec:directional-mismatch}).

\subsection{Motivation: Student-Induced Teacher Hacking}
\label{sec:teacher-hacking}

In this section, we highlight several key observations about the interaction between student-generated trajectories and teacher supervision during OPD; the experimental setup is detailed in \cref{sec:experimental-setup}. In OPD, the teacher evaluates each token conditioned on the prefix generated by the student, rather than their own original rollout alone. This allows the student to influence the teacher's conditional predictions by steering the context, which can lead to an exploitable failure mode: the student can reduce the measured teacher–student divergence by generating contexts that make the teacher's outputs more aligned with its own, without actually improving its reasoning ability. We refer to this phenomenon as \emph{student-induced teacher hacking}.

We first examine the trajectory of the OPD training and find a conspicuous pattern \emph{repetitive generation}: as training proceeds, the student increasingly falls into loops where the same token sequence is repeated many times, e.g., repeating \texttt{I hope it is correct} 149 times in math problems. We then evaluate a fixed set of prompts across training checkpoints and find that the proportion of repetitive rollouts rises from \(16.8\%\) to \(48\%\) (\cref{fig:teacher-hacking-motivation}a). These loops are rarely transient; once the student enters one, it tends to persist for many iterations, yielding the heavy-tailed distribution of repetition counts shown in \cref{fig:teacher-hacking-motivation}b. Notably, neither the prompts nor the objective function explicitly encourage repetition. This behavior is conjectured to emerge organically from the interplay between OPD optimization and the student-conditioned teacher signal.

We next examine how the teacher reacts once the student falls into such a loop. For each naturally occurring repetitive rollout, we isolate its repeating unit and construct controlled prefixes with varying numbers of repetitions, then evaluate all prefixes on the same set of subsequent tokens. As shown in \cref{fig:teacher-hacking-motivation}c, increasing the repetition count simultaneously reduces both the student--teacher KL divergence and the teacher's predictive entropy. With sixteen repetitions, the KL drops by a factor of \(63\times\). Far from correcting the repetition, the teacher grows increasingly confident in continuing it. This pattern is consistent under greedy decoding (see~\cref{app:teacher-continuation}): when prompted with natural loop prefixes from the student, the teacher continues the same repetitive pattern in \(93\%\) of cases. In effect, repetition turns a degenerate student trajectory into apparent teacher--student agreement, precisely eroding the supervisory signal where it is most needed.

\begin{figure*}[t]
    \centering
    \includegraphics[width=\textwidth]{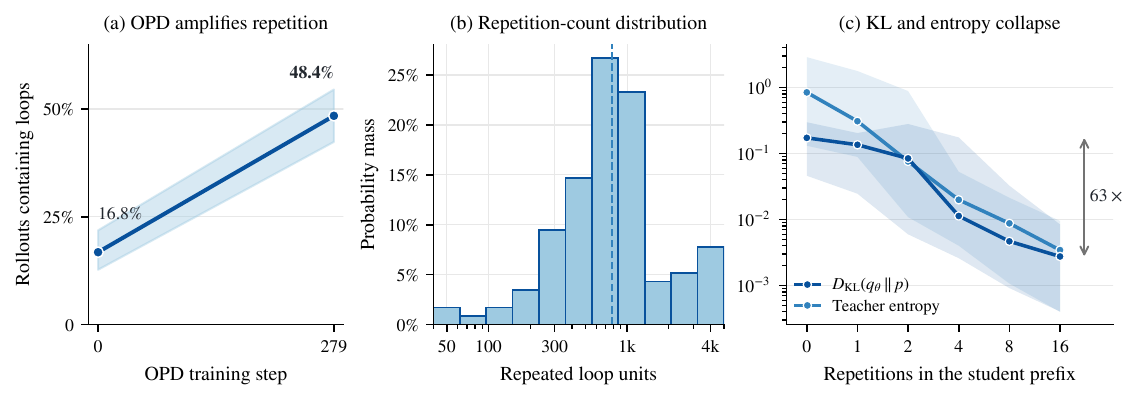}
    \caption{
        \textbf{Repetition emerges as an exploit of teacher conditioning in OPD.}
        (a) On a fixed set of DAPO-Math-17K prompts, the fraction of rollouts containing repetitive loops increases from \(16.8\%\) before training to \(48.4\%\) after OPD training; shaded regions denote \(95\%\) confidence intervals.
        (b) Repetitive rollouts from the final checkpoint exhibit a heavy-tailed distribution of repetition counts.
        (c) Increasing the number of repetitions in a controlled student prefix simultaneously reduces student--teacher KL and teacher entropy, with sixteen repetitions lowering KL by \(63\times\).
    }
    \label{fig:teacher-hacking-motivation}
\end{figure*}

\subsection{Informative mismatch is directional.}
\label{sec:directional-mismatch}
The results above show that teacher--student agreement may not necessarily
constitute useful supervision: the two models can agree even on degenerate
trajectories induced by the student. Our matched-only control experiment provides further evidence for this
distinction: supervising only the $60\%$ best-matched positions raises
the Avg@8 of Qwen3-1.7B-Base from $6.87$ to just $7.18$, whereas
supervising the $20\%$ most mismatched positions more than doubles it to
$14.58$ (see
\cref{app:matched-token-control,tab:matched-token-control-full}). Training on
matched positions yields substantially smaller gains
than supervising a much smaller set of mismatched positions, particularly
under strong teacher--student mismatch. These observations motivate us to
shift the gradient budget away from already matched positions and toward
states where the two distributions meaningfully disagree.

Having established where supervision should be concentrated, we next consider
how different forms of disagreement should be handled. The answer follows
from the asymmetric structure of the reverse-KL objective used by vanilla OPD
in~\eqref{eq:opd}. This divergence is well known to be exclusive and zero-forcing: it aggressively penalizes the student for assigning probability to tokens that the teacher deems unlikely, yet offers comparatively weak pressure to cover teacher-preferred tokens that fall outside the student's effective support \citep{minka2005divergence,gu2024minillm,agarwal2024policy}. 
These two sides of reverse KL introduce a directional asymmetry to teacher–student mismatch in OPD. On one hand, when the student over-weights a token that the teacher assigns low probability, the discrepancy is easy to observe through sampling, but suppressing it can lead to unstable updates. On the other hand, when the student under-weights a token that the teacher favors, the resulting supervision may be highly informative, yet it remains largely inaccessible to standard student sampling.
The two directions are visualized in the top row of
\cref{fig:tide-pipeline}: student excess (top left) is readily observable
but can induce unstable sampled-token corrections, whereas student deficit
(top right) is difficult to observe and requires analytic teacher-top-\(K\)
guidance.

This asymmetry is especially pronounced in sampled-token OPD. A
negative advantage \(a_t<0\) indicates that the student overweights the
sampled token relative to the teacher, and on-policy rollouts expose
such tokens readily, since sampling follows the student's own
distribution. The suppression magnitude
\(-a_t=\log[q_\theta(o_t| s_t)/p(o_t| s_t)]\) measures this gap:
it stays near zero when both models consider a token similarly
unlikely, but grows without bound as the teacher probability
vanishes while the student retains mass~\citep{zhao2026poweropd}.A handful of strongly teacher-rejected tokens can thus dominate the
reverse-KL update and destabilize training. Our first branch locates
the largest gaps along the sampled trajectory and suppresses them with
a bounded signal in place of the raw log-ratio. This observable-but-unstable
regime corresponds to the student-excess case in
\cref{fig:tide-pipeline} (top left).

The opposite direction, however, cannot be addressed by simply rescaling or reshape \(a_t\). Consider a token \(v\) that the teacher assigns high probability but the student deems nearly impossible. Its
sampling-weighted teacher-favoring signal satisfies
\[
    q_\theta(v| s)
    \log\frac{p(v| s)}{q_\theta(v| s)}
    \longrightarrow 0
    \qquad
    \text{as }q_\theta(v| s)\to0.
\]
Thus, the more severely the student under-weights a teacher-preferred token, the less likely OPD is to ever sample it, and hence the less opportunity the algorithm has to reinforce it. Our empirical measurements confirm that these missing modes are not only difficult to observe through standard sampling, but also carry substantial learning value when recovered (\cref{app:deficit-accessibility}). This motivates our second branch: rather than waiting for teacher-preferred tokens to be sampled by chance, we directly inspect the teacher's top-\(K\) support and provide explicit analytic guidance toward these under-covered tokens. This hard-to-sample regime corresponds to the student-deficit case in \cref{fig:tide-pipeline} (top right), where the teacher's top-\(K\) candidates expose alternatives that may never appear in the student rollout.

The directional asymmetry of reverse KL thus calls for a solution that treats the two directions differently. Student-overweighted tokens are easy to spot but require stable suppression; teacher-preferred tokens are rarely seen and need explicit support recovery. The next section turns these two observations into separate suppression and recovery branches.

\section{Methodology}
\label{sec:method}

In this section, we present \method{}, an on-policy distillation method that uses the magnitude
of teacher--student mismatch to determine \emph{where} to update and its
direction to determine \emph{how} to update. 
At each student‑visited state, \method{} diagnoses teacher–student mismatch from two complementary perspectives: one flags tokens where the student over‑estimates probability, and the other detects teacher‑preferred tokens that the student rarely generates. Only positions exhibiting substantial mismatch in either direction receive corrective updates; all remaining tokens are excluded from the distillation objective. The overall
procedure is illustrated in \cref{fig:tide-pipeline}. Its bottom row summarizes
three stages: \method{} computes branch-specific mismatch statistics (bottom
left), independently routes severe positions to the two branches (bottom
center), and applies the corresponding correction objectives (bottom right).

\subsection{Stabilizing Student-Excess Updates}
\label{sec:tide-excess}

We first consider positions where the sampled token receives more probability
from the student than from the teacher, i.e., student-excess tokens. These positions satisfy \(a_t<0\) and can be  ranked by the magnitude of the excess $e_t=[-a_t]_+$, where \([u]_+=\max\{u,0\}\).

Let \(\mathcal B\) denote the set of response positions in a training batch. To better locate the tokens with severe mismatch, we consider a constant fraction 
 \(\rho^-\) of positions ranked by their excess scores. Specifically, we define the corresponding batch-quantile threshold and selection mask as
\[
    \tau^-
    =
    Q_{1-\rho^-}\!\left(\{e_j:j\in\mathcal B\}\right),
    \qquad
    m_t^-
    =
    \mathbbm{1}[a_t<0,\ e_t\geq\tau^-],
\]
where \(Q_\alpha\) denotes the empirical \(\alpha\)-quantile.  
This selection focuses the distillation signal on positions where the teacher assigns low probability to the student's sampled token. 
The score-and-gate construction is shown by the orange route from the
bottom-left statistics panel to the retained student-excess region in the
bottom-center panel of \cref{fig:tide-pipeline}.

As discussed above, the teacher assigns substantially less probability
than the student to tokens selected by the excess gate, making $a_t$
strongly negative. The resulting log-ratio penalty
$-a_t=\log[q_\theta(o_t| s_t)/p(o_t| s_t)]$ grows without bound
as the teacher probability approaches zero. Thus, directly applying the original OPD advantage after selection would  concentrate the gradient too heavily on a small number of extreme corrections, leading to the training instability. 

We address this problem by applying a bounded, monotonic transformation
to the raw advantage at each selected position. Clipping-style
functions (e.g., hard clipping or tanh) would also bound the update,
but they rely on an arbitrary saturation scale, collapse all severe
discrepancies to nearly identical corrections, and yield an update that
no longer corresponds to any well-defined divergence between the two
distributions. We instead adopt the Hellinger-shaped function, which
bounds the update and locally preserves the log-ratio while remaining
the exact gradient of a proper divergence:
\[
    h(a_t)
    =
    2[\exp(a_t/2)-1].
\]
Accordingly, we take the following the policy gradient in OPD:
\[
   g_{\mathrm{exc}}
    =
    -\mathbb E\bigg[
        \sum_t
        m_t^-\,
        \sg[h(a_t)]
        \log q_\theta(o_t| s_t)
    \bigg].
\]
Thus, quantile selection determines where suppression is informative, while
Hellinger shaping prevents the selected reverse-KL tail from destabilizing
optimization. The resulting bounded suppression objective appears in
\cref{fig:tide-pipeline} (bottom right).

The following proposition explains the critical properties of  Hellinger transformation.
\begin{proposition}[Bounded, locally faithful Hellinger shaping]
\label{prop:hellinger}
For a student-excess token, \(a<0\),
\begin{equation*}
-2<h(a)<0,
\qquad
h(a)=a+O(a^2)\quad\text{as }a\to0.
\end{equation*}
Moreover, before position selection, a token sampled from the student gives
the following unbiased gradient identity when \(h(a(v))\) is treated as a
detached coefficient:
\begin{equation}
\label{eq:hellinger-unbiased}
\mathbb E_{v\sim q_\theta}
\!\left[-h(a(v))\nabla_\theta\log q_\theta(v| s)\right]
=4\nabla_\theta H^2(p,q).
\end{equation}
\end{proposition}
\Cref{prop:hellinger} explains the advantage of Hellinger shaping compared to the simple 
clipping of the raw advantage. Its negative side is bounded, preventing
teacher-rejected tokens from producing arbitrarily large suppression updates.
At the same time, $h(a)\approx a$, i.e., in the regime near agreement, so moderate corrections retain the first-order behavior of standard OPD. Moreover, \eqref{eq:hellinger-unbiased} shows that this transformation arises naturally from a proper divergence, grounding the design in a principled objective rather than an ad hoc heuristic. 


\subsection{Recovering Under-Covered Teacher Support}
\label{sec:tide-deficit}

Unlike the excess case, missing teacher support cannot be detected from the sampled advantage $a_t$, which only compares the two distributions at the student's chosen token. To identify teacher-preferred tokens that the student overlooks, we inspect the teacher's top-$K$ support. Define the top-K support and the renormalized teacher distribution within this support as follows:
\[
\mathcal K_t = \operatorname{TopK}_{v\in\mathcal V}\, p(v| s_t); \quad \bar p_t(v) = \frac{p(v| s_t)}{\sum_{u\in\mathcal K_t} p(u| s_t)}, \  v\in\mathcal K_t.
\]
Then we quantify the student's deficit on this support based on the following metric
\begin{equation}
\label{eq:tide-deficit-score}
d_t = \sum_{v\in\mathcal K_t} \bar p_t(v) \log \frac{\bar p_t(v)}{q_\theta(v| s_t)}.
\end{equation}
This is the blue deficit statistic shown in the bottom-left panel of
\cref{fig:tide-pipeline}, motivated by the teacher-preferred alternatives in
the top-right panel.
Notably, unlike standard top-$K$ OPD~\citep{jin2026entropy}, we do not renormalize the student probabilities over $\mathcal K_t$. The score $d_t$ quantifies how well the student covers the teacher's preferred continuations, and enjoys a natural interpretation: it is nonnegative, vanishes only when the student matches $\bar p_t$ within $\mathcal K_t$ and assigns all its probability mass to this set, and grows large whenever the student underweights tokens that the teacher deems plausible. These properties are formalized in \Cref{prop:deficit-decomposition} below.

\begin{proposition}[Deficit-score decomposition]
\label{prop:deficit-decomposition}
Let \(Q_t=\sum_{v\in\mathcal K_t}q_\theta(v| s_t)\) denote the total
student mass on \(\mathcal K_t\), and let
\(q_t^{\mathcal K}(v)=q_\theta(v| s_t)/Q_t\) be the corresponding
renormalized distribution for \(v\in\mathcal K_t\). Then the deficit
score in~\eqref{eq:tide-deficit-score} satisfies
\[
d_t
=
\underbrace{D_{\mathrm{KL}}
    \bigl(\bar p_t\|q_t^{\mathcal K}\bigr)}
    _{\text{mismatch within }\mathcal K_t}
+
\underbrace{(-\log Q_t)}
    _{\text{insufficient mass on }\mathcal K_t}
\ge 0.
\]
Equality holds if and only if \(q_t^{\mathcal K}=\bar p_t\) and \(Q_t=1\).
\end{proposition}

Additionally, this decomposition reveals two complementary failure modes: the KL term captures misallocation of probability among the teacher's leading candidates, while $-\log Q_t$ measures how much student mass falls outside this set. This also justifies why we do not renormalize the student distribution in~\eqref{eq:tide-deficit-score}: renormalization would eliminate the coverage gap encoded by $-\log Q_t$, making the score blind to the very deficiency we aim to correct.

A large $d_t$ signals that the student poorly covers the teacher's preferred continuations at that state. We therefore use $d_t$ as a priority score and concentrate the deficit branch on the most severe positions. Within each batch, we retain the largest $\rho^+$ fraction of scores:
\[
\tau^+
=
\operatorname{Quantile}_{1-\rho^+}
\bigl(\{d_j:j\in\mathcal B\}\bigr),
\qquad
m_t^+=\mathbbm{1}[d_t\ge\tau^+],
\]
and denote the active set as $\mathcal A^+=\{t:m_t^+=1\}$. For these selected positions, we do not rely on sampling to surface the missing tokens. Instead, we directly distill the teacher's renormalized top-$K$ distribution:
\begin{equation}
\label{eq:tide-deficit-loss}
\mathcal L_{\mathrm{def}}
=
-\mathbb E\left[
\sum_t m_t^+
\sum_{v\in\mathcal K_t}
\bar p_t(v)\log q_\theta(v| s_t)
\right].
\end{equation}
This completes the blue route in \cref{fig:tide-pipeline}: the deficit score
selects positions in the bottom-center panel, and the analytic top-\(K\)
objective performs probability-mass recovery in the bottom-right panel.
The underlying principle is that student rollouts determine where to learn, but not which tokens to learn: the deficit branch uses the student's mass on the teacher's top-$K$ set to detect missing support and directly distills those candidates when coverage is insufficient.

\subsection{Joint Objective}
\label{sec:tide-objective}

The excess and deficit branches operate on the same student-generated
states but address complementary forms of mismatch. The two routing decisions
in \cref{fig:tide-pipeline} (bottom center) are independent: a position may
activate the excess branch, the deficit branch, or both, while positions
selected by neither branch receive no distillation update. We optimize the
two branches jointly as
\begin{equation}
\label{eq:tide-joint-objective}
\mathcal L_{\mathrm{\method}}
=
\mathcal L_{\mathrm{exc}}
+
\lambda\mathcal L_{\mathrm{def}},
\end{equation}
where $\lambda$ controls the strength of teacher-support recovery
relative to excess suppression. Each branch contributes only at the
positions selected by its corresponding gate; positions selected by
neither branch receive no distillation update. 
\Cref{alg:tide} summarizes the resulting training procedure. 
\IncMargin{1.em}
\begin{algorithm}[t]
\caption{\method{} Training}
\label{alg:tide}
\small
\SetNlSkip{0.5em}
\DontPrintSemicolon

\KwIn{
prompt distribution \(\mathcal D\), student \(q_\theta\), teacher \(p\),
top-\(K\) size \(K\), keep ratios \(\rho^-,\rho^+\),
recovery weight \(\lambda\), learning rate \(\eta\)
}
\KwOut{Trained student parameters \(\theta\)}

\While{not converged}{
    Sample a prompt batch \(x\sim\mathcal D\) and generate
    \(o\sim q_\theta(\cdot| x)\)\;

    Let \(\mathcal B\) be all valid response-token positions and
    \(N=|\mathcal B|\)\;

    \ForEach{\(t\in\mathcal B\)}{
        Construct \(s_t=(x,o_{<t})\) and compute
        \(a_t\leftarrow
        \log p(o_t| s_t)-\log q_\theta(o_t| s_t)\),
        \(e_t\leftarrow[-a_t]_+\)\;

        Obtain
        \(\mathcal K_t\leftarrow
        \operatorname{TopK}_{v}p(v| s_t)\)
        and normalize
        \(\bar p_t(v)\leftarrow
        p(v| s_t)/
        \sum_{u\in\mathcal K_t}p(u| s_t)\)\;

        Compute
        \(d_t\leftarrow
        \sum_{v\in\mathcal K_t}
        \bar p_t(v)
        \log
        \frac{\bar p_t(v)}
             {q_\theta(v| s_t)}\)\;
    }

    Set
    \(\tau^-\leftarrow
    Q_{1-\rho^-}(\{e_t:t\in\mathcal B\})\)
    and
    \(\tau^+\leftarrow
    Q_{1-\rho^+}(\{d_t:t\in\mathcal B\})\)\;

    Select
    \(\mathcal A^-\leftarrow
    \{t:a_t<0,\ e_t\geq\tau^-\}\)
    and
    \(\mathcal A^+\leftarrow
    \{t:d_t\geq\tau^+\}\)\;

    Compute
    \[
        \mathcal L_{\mathrm{exc}}
        \leftarrow
        -\frac{1}{N}
        \sum_{t\in\mathcal A^-}
        \operatorname{sg}\!\left[2(e^{a_t/2}-1)\right]
        \log q_\theta(o_t| s_t)
    \]

    Compute
    \[
        \mathcal L_{\mathrm{def}}
        \leftarrow
        -\frac{1}{N}
        \sum_{t\in\mathcal A^+}
        \sum_{v\in\mathcal K_t}
        \bar p_t(v)\log q_\theta(v| s_t)
    \]

    Set
    \(\mathcal L_{\mathrm{\method}}
      \leftarrow
      \mathcal L_{\mathrm{exc}}
      +\lambda\mathcal L_{\mathrm{def}}\)\;

    Update
    \(\theta\leftarrow
      \theta-\eta\nabla_\theta\mathcal L_{\mathrm{\method}}\)\;
}

\Return{\(\theta\)}\;

\end{algorithm}
\DecMargin{1.em}

\section{Experiments}

\subsection{Experimental Setup}
\label{sec:experimental-setup}

\paragraph{Models and baselines.}
We employ two teacher--student pairs with different degrees of model mismatch.
The first uses JustRL-DeepSeek-1.5B~\citep{he2025justrl} as the teacher and
DeepSeek-R1-Distill-Qwen-1.5B~\citep{guo2025deepseek} as the student, while
the second uses Qwen3-8B as the teacher and Qwen3-1.7B-Base as the
student~\citep{yang2025qwen3}.
Following~\citet{li2026rethinking}, we characterize the initial
teacher--student mismatch using the overlap between their top-$K$
candidate sets on student-generated prefixes. A higher initial overlap indicates a smaller teacher--student mismatch,
while a lower overlap indicates a larger mismatch. Under this criterion, the
JustRL-DeepSeek-1.5B $\rightarrow$ DeepSeek-R1-Distill-Qwen-1.5B pair
constitutes our \emph{weak-mismatch} setting, while the
Qwen3-8B $\rightarrow$ Qwen3-1.7B-Base pair constitutes our
\emph{strong-mismatch} setting.

We compare \method{} against GRPO
~\citep{shao2024deepseekmath}, standard OPD
~\citep{lu2025onpolicydistillation}, FiRe-OPD~\citep{li2026filter},
PowerOPD~\citep{zhao2026poweropd}, and AOPD~\citep{jia2026asymmetric}. PowerOPD applies a bounded power
transformation to stabilize the sampled-token reward, whereas AOPD preserves
positive reinforcement while replacing non-positive updates with localized
divergence minimization. We additionally report the original teacher and
student models as reference points.

\paragraph{Training Details.}
We train all methods for one epoch on DAPO-Math-17K
\citep{yu2026dapo} using the verl framework
\citep{sheng2024hybridflow}. Each prompt produces four on-policy
student responses, and all methods use the same student initialization,
data order, rollout budget, and optimizer configuration. For
\method{}, we set $K=16$, $\rho^-=\rho^+=0.2$, and $\lambda=1.0$.
Full training, implementation, baseline, and evaluation details are
provided in \cref{app:experimental_details}.

\subsection{Main Results}


\begin{table}[t]
\centering
\caption{\textbf{Avg@8 results across different teacher--student mismatch
settings} (\%). Each model generates eight responses per problem.
\emph{Student} and \emph{Teacher} denote reference models and are excluded
from bolding. \textbf{Bold} denotes the best trained model in each column.
\(\Delta\) denotes the improvement of \method{} over OPD.}
\label{tab:avg8}
\footnotesize
\setlength{\tabcolsep}{3.5pt}
\renewcommand{\arraystretch}{1.15}
\begin{adjustbox}{max width=\textwidth}
\begin{tabular}{lccccccccc|c}
\toprule
Method
& AIME'24 & AIME'25 & AMC23 & MATH
& Minerva & Olymp. & BRUMO & CMIMC & HMMT & Avg. \\
\midrule

\rowcolor{groupbg}
\multicolumn{11}{l}{
\emph{Weak mismatch:} JustRL-DeepSeek-1.5B
\(\rightarrow\) DeepSeek-R1-Distill-Qwen-1.5B} \\

Student (Base)
& 27.1 & 24.2 & 73.1 & 83.4 & 28.0
& 44.3 & 27.1 & 14.4 & 12.1 & 37.1 \\

Teacher
& 53.3 & 40.0 & 87.5 & 86.3 & 33.5
& 54.0 & 43.3 & 22.8 & 22.1 & 49.2 \\

\addlinespace[2pt]

GRPO
& 22.5 & 21.2 & 72.5 & 83.9 & 29.5
& 46.3 & 33.3 & 16.6 & 14.2 & 37.8 \\

OPD
& 43.8 & 32.1 & 82.2 & 86.2 & \textbf{34.1}
& 51.6 & 40.4 & 21.9 & 19.2 & 45.7 \\

FiRe-OPD~\citep{li2026filter}
& 39.2 & \textbf{37.1} & \textbf{84.7} & 85.3
& 33.2 & 52.3 & 42.5 & 21.6 & \textbf{23.3} & 46.6 \\

PowerOPD~\citep{zhao2026poweropd}
& \textbf{60.0} & 31.2 & 81.6 & 84.0 & 32.6
& 50.6 & 38.3 & 19.7 & 20.4 & 46.5 \\

AOPD~\citep{jia2026asymmetric}
& 50.8 & 33.3 & 81.9 & 85.4 & 33.0
& 51.3 & 41.2 & 20.6 & 17.9 & 46.2 \\

\rowcolor{oursbg}
\method{}
& 45.8 & 33.3 & 84.4 & \textbf{86.7}
& 33.5 & \textbf{52.5} & \textbf{42.9} & \textbf{22.2}
& 19.2 & \textbf{46.7} \\

\(\Delta\) vs.\ OPD
& \gain{+2.0} & \gain{+1.2} & \gain{+2.2}
& \gain{+0.5} & \drop{\(-0.6\)} & \gain{+0.9}
& \gain{+2.5} & \gain{+0.3} & 0.0 & \gain{+1.0} \\

\midrule

\rowcolor{groupbg}
\multicolumn{11}{l}{
\emph{Strong mismatch:} Qwen3-8B
\(\rightarrow\) Qwen3-1.7B-Base} \\

Student (Base)
& 3.3 & 0.0 & 13.4 & 23.8 & 6.3
& 9.8 & 5.0 & 0.0 & 0.0 & 6.9 \\

Teacher
& 23.3 & 23.3 & 68.8 & 83.3 & 28.0
& 49.8 & 29.2 & 11.6 & 11.7 & 36.6 \\

\addlinespace[2pt]

GRPO
& 3.3 & 0.4 & 35.0 & 55.5 & 18.7
& 23.3 & 9.6 & 0.9 & 0.0 & 16.3 \\

OPD
& 2.9 & 1.2 & 10.3 & 26.5 & 6.9
& 10.7 & 2.9 & 0.3 & 0.0 & 6.9 \\

FiRe-OPD~\citep{li2026filter}
& 3.3 & 4.2 & 31.9 & 52.2 & 14.2
& 19.7 & 7.1 & 0.9 & \textbf{0.8} & 14.9 \\

PowerOPD~\citep{zhao2026poweropd}
& 6.7 & 4.2 & 26.2 & 55.1 & 15.1
& 24.5 & 7.1 & 1.9 & 0.4 & 15.7 \\

AOPD~\citep{jia2026asymmetric}
& 3.3 & \textbf{4.6} & 34.4 & 62.3 & 21.2
& 25.5 & \textbf{10.8} & 1.9 & 0.4 & 18.3 \\

\rowcolor{oursbg}
\method{}
& \textbf{10.0} & \textbf{4.6} & \textbf{39.4} & \textbf{65.1}
& \textbf{21.8} & \textbf{28.6} & 10.4
& \textbf{2.8} & 0.4 & \textbf{20.3} \\

\(\Delta\) vs.\ OPD
& \gain{+7.1} & \gain{+3.4} & \gain{+29.1}
& \gain{+38.6} & \gain{+14.9} & \gain{+17.9}
& \gain{+7.5} & \gain{+2.5} & \gain{+0.4}
& \gain{+13.4} \\

\bottomrule
\end{tabular}
\end{adjustbox}
\end{table}

\begin{table}[t]
\centering
\caption{\textbf{Pass@8 results across different teacher--student mismatch
settings} (\%). Pass@8 is the percentage of problems for which at least one
of eight generated responses is correct. \emph{Student} and \emph{Teacher}
denote reference models and are excluded from bolding. \textbf{Bold} denotes
the best trained model in each column. \(\Delta\) denotes the improvement of
\method{} over OPD.}
\label{tab:pass8}
\footnotesize
\setlength{\tabcolsep}{3.5pt}
\renewcommand{\arraystretch}{1.15}
\begin{adjustbox}{max width=\textwidth}
\begin{tabular}{lccccccccc|c}
\toprule
Method
& AIME'24 & AIME'25 & AMC23 & MATH
& Minerva & Olymp. & BRUMO & CMIMC & HMMT & Avg. \\
\midrule

\rowcolor{groupbg}
\multicolumn{11}{l}{
\emph{Weak mismatch:} JustRL-DeepSeek-1.5B
\(\rightarrow\) DeepSeek-R1-Distill-Qwen-1.5B} \\

Student (Base)
& 36.7 & 33.3 & 95.0 & 95.0 & 43.4
& 60.8 & 56.7 & 30.0 & 20.0 & 52.3 \\

Teacher
& 53.3 & 40.0 & 92.5 & 94.2 & 45.2
& 68.2 & 66.7 & 42.5 & 36.7 & 59.9 \\

\addlinespace[2pt]

GRPO
& 43.3 & 36.7 & 90.0 & 94.4 & 45.6
& 64.8 & 63.3 & 32.5 & 30.0 & 55.6 \\

OPD
& 76.7 & 53.3 & 95.0 & 94.6 & 44.9
& 66.8 & 56.7 & 35.0 & 33.3 & 61.8 \\

FiRe-OPD~\citep{li2026filter}
& \textbf{80.0} & \textbf{56.7} & 95.0 & 94.2 & 44.1
& 68.0 & 60.0 & 35.0 & \textbf{43.3} & 64.0 \\

PowerOPD~\citep{zhao2026poweropd}
& 60.0 & 43.3 & \textbf{97.5} & 94.4 & \textbf{46.3}
& 67.2 & 60.0 & 32.5 & 33.3 & 59.4 \\

AOPD~\citep{jia2026asymmetric}
& 73.3 & 43.3 & 95.0 & \textbf{94.8} & 46.0
& 67.4 & 63.3 & 40.0 & 26.7 & 61.1 \\

\rowcolor{oursbg}
\method{}
& \textbf{80.0} & 50.0 & 95.0 & 94.6 & 45.2
& \textbf{68.1} & \textbf{66.7} & \textbf{42.5}
& \textbf{43.3} & \textbf{65.0} \\

\(\Delta\) vs.\ OPD
& \gain{+3.3} & \drop{\(-3.3\)} & 0.0
& 0.0 & \gain{+0.3} & \gain{+1.3}
& \gain{+10.0} & \gain{+7.5} & \gain{+10.0}
& \gain{+3.2} \\

\midrule

\rowcolor{groupbg}
\multicolumn{11}{l}{
\emph{Strong mismatch:} Qwen3-8B
\(\rightarrow\) Qwen3-1.7B-Base} \\

Student (Base)
& 3.3 & 0.0 & 40.0 & 68.2 & 24.3
& 34.3 & 20.0 & 0.0 & 0.0 & 21.1 \\

Teacher
& 33.3 & 43.3 & 92.5 & 94.2 & 39.7
& 67.7 & 60.0 & 32.5 & 26.7 & 54.4 \\

\addlinespace[2pt]

GRPO
& 6.7 & 3.3 & 60.0 & 83.4 & 36.0
& 45.5 & 23.3 & 7.5 & 0.0 & 29.5 \\

OPD
& 13.3 & 13.3 & 45.0 & 73.8 & 28.7
& 37.7 & 13.3 & 2.5 & 0.0 & 25.3 \\

FiRe-OPD~\citep{li2026filter}
& 3.3 & 10.0 & 55.0 & 77.6 & 32.7
& 40.9 & 16.7 & 5.0 & \textbf{3.3} & 27.2 \\

PowerOPD~\citep{zhao2026poweropd}
& 6.7 & 10.0 & 57.5 & 83.6 & 37.1
& 47.2 & 23.3 & \textbf{12.5} & \textbf{3.3} & 31.2 \\

AOPD~\citep{jia2026asymmetric}
& 3.3 & \textbf{16.7} & \textbf{65.0} & 86.0 & 40.4
& 47.8 & \textbf{33.3} & 10.0 & \textbf{3.3}
& 34.0 \\

\rowcolor{oursbg}
\method{}
& \textbf{16.7} & 10.0 & 60.0 & \textbf{86.2}
& \textbf{41.9} & \textbf{50.3} & 26.7
& \textbf{12.5} & \textbf{3.3} & \textbf{34.2} \\

\(\Delta\) vs.\ OPD
& \gain{+3.4} & \drop{\(-3.3\)} & \gain{+15.0}
& \gain{+12.4} & \gain{+13.2} & \gain{+12.6}
& \gain{+13.4} & \gain{+10.0} & \gain{+3.3}
& \gain{+8.9} \\

\bottomrule
\end{tabular}
\end{adjustbox}
\end{table}

\paragraph{Overall Performance.}
Tables~\ref{tab:avg8} and~\ref{tab:pass8} reveal a clear
mismatch-dependent pattern. Under weak teacher--student mismatch, the
distillation methods achieve broadly comparable Avg@8, while \method{}
obtains the highest overall Avg@8 and Pass@8 of \(46.7\%\) and \(65.0\%\),
improving over OPD by \(1.0\) and \(3.2\) points, respectively. The
differences become substantially larger under strong mismatch. Standard OPD
falls to the original student's Avg@8 level of \(6.9\%\), whereas \method{}
reaches \(20.3\%\), outperforming OPD, FiRe-OPD, GRPO, PowerOPD, and AOPD by
\(13.4\), \(5.4\), \(4.0\), \(4.6\), and \(2.0\) points, respectively. For
Pass@8, \method{} obtains \(34.2\%\), compared with \(25.3\%\) for OPD,
\(27.2\%\) for FiRe-OPD, \(29.5\%\) for GRPO, \(31.2\%\) for PowerOPD, and
\(34.0\%\) for AOPD. Moreover, \method{} achieves the best-or-tied Avg@8 on
seven of nine benchmarks and the best-or-tied Pass@8 on six of nine
benchmarks under strong mismatch. These results are consistent with our
analysis: sampled-token corrections are often sufficient when the teacher
and student distributions substantially overlap, but analytic recovery of
teacher-preferred alternatives becomes increasingly important as mismatch
grows and the student's sampled support becomes less informative.

\begin{table}[t]
\centering
\caption{\textbf{Generation diagnostics across mismatch settings.}
Avg@8 and response length are macro-averaged over nine benchmarks.
Missing \texttt{\textbackslash boxed} denotes the percentage of all
generated responses that do not contain a boxed final answer.
\textbf{Bold} denotes the best trained model in each row; response
length is not bolded, since shorter generations are not inherently
preferable.}
\label{tab:generation_diagnostics}
\small
\setlength{\tabcolsep}{7pt}
\renewcommand{\arraystretch}{1.15}
\begin{tabular}{l ccccc >{\columncolor{oursbg}}c}
\toprule
Metric & GRPO & OPD & FiRe-OPD & PowerOPD & AOPD & \method{} \\
\midrule

\rowcolor{groupbg}
\multicolumn{7}{l}{\emph{Weak mismatch:} JustRL-DeepSeek-1.5B
$\rightarrow$ DeepSeek-R1-Distill-Qwen-1.5B} \\
Avg@8 (\%)
& 37.8 & 45.7 & 46.6 & 46.5 & 46.2 & \textbf{46.7} \\
Avg.\ length
& 8{,}048 & 7{,}657 & 8{,}055 & 8{,}658 & 7{,}982 & 7{,}571 \\
Missing \texttt{\textbackslash boxed} (\%)
& \textbf{1.8} & 4.2 & 4.2 & 5.3 & 4.8 & 3.7 \\

\midrule

\rowcolor{groupbg}
\multicolumn{7}{l}{\emph{Strong mismatch:} Qwen3-8B
$\rightarrow$ Qwen3-1.7B-Base} \\
Avg@8 (\%)
& 16.3 & 6.9 & 14.9 & 15.7 & 18.3 & \textbf{20.3} \\
Avg.\ length
& 5{,}261 & 22{,}395 & 28{,}307 & 13{,}972 & 5{,}894 & 7{,}294 \\
Missing \texttt{\textbackslash boxed} (\%)
& 17.1 & 65.5 & 54.0 & 26.3 & \textbf{4.7} & 5.4 \\

\bottomrule
\end{tabular}
\end{table}




\paragraph{Generation Stability.}
\Cref{tab:generation_diagnostics} shows that the accuracy gap is accompanied
by substantially different generation behavior. Under strong mismatch, OPD
and FiRe-OPD produce \(22{,}395\)- and \(28{,}307\)-token responses, with
\(65.5\%\) and \(54.0\%\) missing a boxed answer; PowerOPD reduces but does
not eliminate this degeneration, averaging \(13{,}972\) tokens with a
\(26.3\%\) missing-answer rate. In contrast, \method{} achieves the highest
Avg@8 of \(20.3\%\) while reducing response length to \(7{,}294\) tokens and
the missing-answer rate to \(5.4\%\). This gain cannot be attributed to
shorter responses alone: GRPO and AOPD generate even shorter outputs but
remain \(4.0\) and \(2.0\) Avg@8 points behind \method{}, respectively.
\subsection{Ablation Studies}

\paragraph{Component Ablation.}
\Cref{tab:component-ablation,fig:component-ablation} shows that the three
components play complementary roles. Selection alone raises Avg@8 from
\(6.87\) to \(14.58\), confirming the value of concentrating supervision on
informative positions. However, it also increases the macro-averaged response
length from \(22.4\)K to \(29.7\)K tokens, indicating that selection alone
does not prevent generation degeneration. Applying bounded Hellinger shaping
to the selected excess positions further improves Avg@8 to \(18.57\) and
partially controls the length increase, reducing it to \(23.6\)K. Analytic
deficit recovery independently reaches a similar Avg@8 of \(18.32\), while
producing substantially shorter responses of \(13.9\)K tokens. Combining all
three components in \method{} achieves the best Avg@8 and Pass@8 of \(20.34\)
and \(34.17\), respectively, while further reducing response length to
\(7.3\)K tokens. Overall, selection determines \emph{where} to learn,
Hellinger shaping stabilizes excess suppression, and deficit recovery restores
teacher-preferred continuations that student sampling tends to miss.

\begin{figure}[t]
\centering
\includegraphics[width=\linewidth]{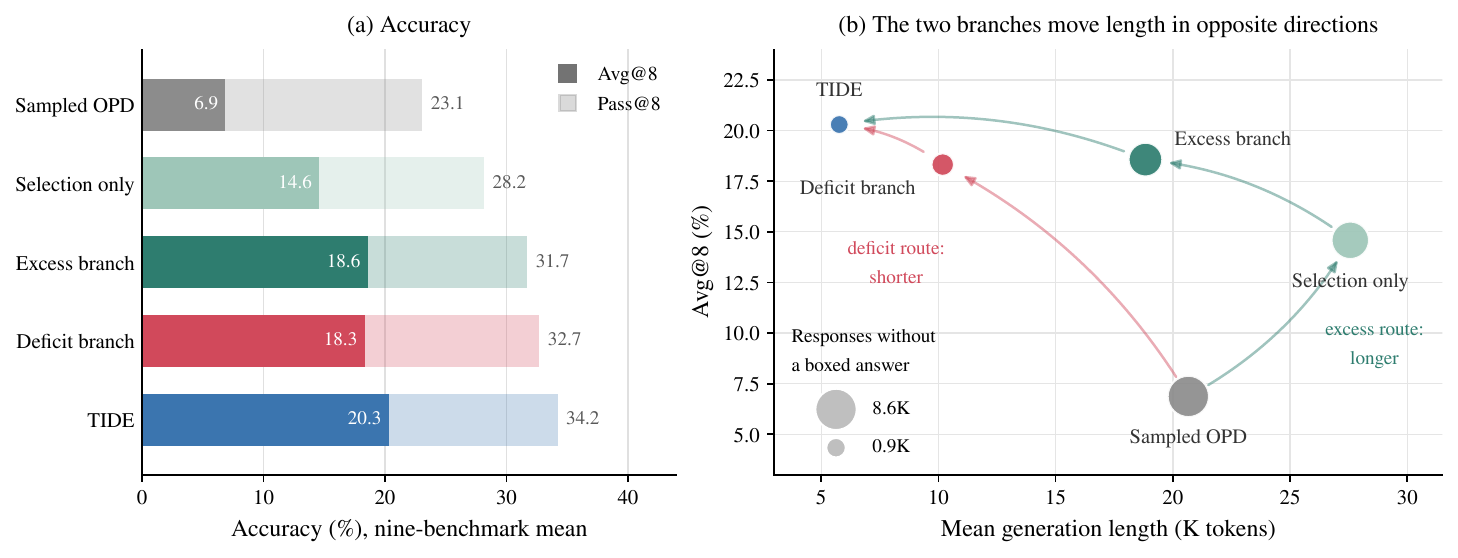}
\caption{\textbf{Component ablation on
Qwen3-8B\,$\to$\,Qwen3-1.7B-Base.}
\textbf{(a)} Avg@8 (solid) and Pass@8 (shaded), macro-averaged over nine
benchmarks. \textbf{(b)} Accuracy versus macro-averaged response length;
marker area denotes the number of responses without a parseable final answer.
The excess and deficit branches provide complementary improvements, and their
combination achieves the highest accuracy with the shortest responses.}
\label{fig:component-ablation}
\end{figure}





\begin{table}[t]
\centering
\caption{\textbf{Component ablation on
Qwen3-8B\,$\to$\,Qwen3-1.7B-Base.}
Avg@8, Pass@8, and response length are macro-averaged over nine
benchmarks.}
\label{tab:component-ablation}
\small
\setlength{\tabcolsep}{5pt}
\renewcommand{\arraystretch}{1.12}
\begin{tabular}{lccc|ccc}
\toprule
Method
& Selection & Hellinger & Deficit
& Avg@8 & Pass@8 & Length \\
\midrule

Sampled OPD
& -- & -- & --
& 6.87 & 25.29 & 22.4K \\

Selection only
& \checkmark & -- & --
& 14.58 & 28.16 & 29.7K \\

Excess branch
& \checkmark & \checkmark & --
& 18.57 & 31.72 & 23.6K \\

Deficit branch
& -- & -- & \checkmark
& 18.32 & 32.71 & 13.9K \\

\rowcolor{oursbg}
\method{}
& \checkmark & \checkmark & \checkmark
& \textbf{20.34} & \textbf{34.17} & \textbf{7.3K} \\

\bottomrule
\end{tabular}
\end{table}
\paragraph{Sensitivity to Hyperparameters.}

\Cref{tab:lambda_sensitivity} studies the deficit-recovery weight
\(\lambda\) with fixed gates \(\rho^-=\rho^+=0.2\). Accuracy remains
relatively stable around the default setting: \(\lambda=0.25\), \(0.5\),
and \(1.0\) achieve similar Avg@8 scores. Their generation behavior,
however, differs substantially. Reducing \(\lambda\) from \(1.0\) to
\(0.25\) increases the macro-averaged response length from \(7.3\)K to
\(26.9\)K tokens and format errors from \(706\) to \(5{,}510\), indicating
that insufficient deficit recovery allows degenerate continuations to
persist. Increasing \(\lambda\) to \(2.0\) further shortens responses but
reduces Avg@8 to \(19.20\). We therefore use \(\lambda=1.0\) as a balanced
default. Sensitivity to the gate keep-rate is reported in
\cref{tab:keep-rate-sensitivity}.

\begin{table}[t]
\centering
\small
\setlength{\tabcolsep}{5pt}
\begin{tabular}{lccccc}
\toprule
\(\lambda\)
& Avg@8 \(\uparrow\)
& Pass@8 \(\uparrow\)
& Avg. Len. \(\downarrow\)
& Format Err. \(\downarrow\)
& Dist.-4 \(\uparrow\) \\
\midrule
\(0.25\) & \(20.09\) & \(32.60\) & \(26.9\)K & \(5{,}510\) & \(0.145\) \\
\(0.5\)  & \(20.35\) & \(36.68\) & \(20.0\)K & \(2{,}488\) & \(0.128\) \\
\(1.0\)  & \(20.34\) & \(34.17\) & \(7.3\)K  & \(706\)     & \(0.305\) \\
\(2.0\)  & \(19.20\) & \(35.41\) & \(5.1\)K  & \(482\)     & \(0.417\) \\
\bottomrule
\end{tabular}
\caption{
Sensitivity to the deficit-recovery weight \(\lambda\) with fixed gates
\(\rho^-=\rho^+=0.2\). Avg@8, Pass@8, response length, and Distinct-4
are macro-averaged over nine benchmarks; format errors are totals.
}
\label{tab:lambda_sensitivity}
\end{table}
\section{Conclusion}

We identify \emph{degenerate agreement} as a failure mode of on-policy
distillation: local token-level matching can mask globally flawed responses.
Teacher--student mismatch is directional: student-excess tokens induce
unstable corrections, whereas valuable student-deficit tokens are rarely
sampled. TIDE addresses both with bounded excess correction and analytic
deficit recovery. Across nine mathematical reasoning benchmarks, it
outperforms strong baselines---by \(13.4\) Avg@8 points over OPD under severe
mismatch---while generating shorter, better-formed responses.

TIDE assumes a locally reliable teacher. Our evaluation is limited to two
teacher--student pairs and mathematical reasoning; dialogue, code, and
multilingual settings remain open. Moreover, our fixed-state, operator-level
theory gives no global convergence guarantee, and adaptive keep-rate schedules
warrant study.

Overall, reliable on-policy distillation should allocate supervision according to both the \emph{direction} and \emph{accessibility} of disagreement, a principle that applies whenever sampling leaves useful corrective signals uncovered.

\bibliography{references}
\bibliographystyle{preprint}

\clearpage
\appendix
\begingroup
\setcounter{tocdepth}{1}
\tableofcontents
\endgroup
\clearpage
\section{Proofs and Additional Theoretical Details}
\label{app:theory-proofs}

This appendix proves the two propositions in \cref{sec:method}.

\paragraph{Notation and assumptions.}
We work at a fixed rollout position \(t\) with student-visited state \(s=s_t\). We abbreviate \(p(v)=p(v|s)\) and \(q(v)=q_\theta(v|s)\), and write \(a(v)=\log[p(v)/q(v)]\) and \(h(a)=2(e^{a/2}-1)\).
We use the normalized squared Hellinger distance
\[
H^2(p,q)=1-\sum_{v\in\mathcal V}\sqrt{p(v)q(v)}.
\]
For the deficit branch, \(\mathcal K_t\), \(\bar p_t\), \(Q_t\), and
\(q_t^{\mathcal K}\) follow the definitions in
\cref{sec:tide-deficit}. We assume finite-vocabulary softmax
distributions with full support. When differentiating a gated loss, its
routing decisions are treated as fixed.

\subsection{Proof of Proposition~\ref{prop:hellinger}}
\label{app:proof-hellinger}

\begin{proof}
For \(a<0\), \(e^{a/2}\in(0,1)\), so \(-2<h(a)<0\). The local
expansion follows from the first-order expansion of the smooth function
\(h\) at zero, with \(h(0)=0\) and \(h'(0)=1\).

By the definition of \(H^2\),
\[
\nabla_\theta H^2(p,q)
=-\frac12\sum_{v\in\mathcal V}
\sqrt{\frac{p(v)}{q(v)}}\,\nabla_\theta q(v).
\]
Using \(\sum_v\nabla_\theta q(v)=0\), we obtain
\begin{align*}
\mathbb E_{v\sim q}
\!\left[-h(a(v))\nabla_\theta\log q(v)\right]
&=-2\sum_v
\left(\sqrt{\frac{p(v)}{q(v)}}-1\right)
\nabla_\theta q(v)\\
&=-2\sum_v\sqrt{\frac{p(v)}{q(v)}}\,
\nabla_\theta q(v)\\
&=4\nabla_\theta H^2(p,q),
\end{align*}
which proves \cref{eq:hellinger-unbiased}.
\end{proof}

\subsection{Proof of Proposition~\ref{prop:deficit-decomposition}}
\label{app:proof-deficit-decomposition}

\begin{proof}
By the definitions in \cref{prop:deficit-decomposition},
\(q(v)=Q_tq_t^{\mathcal K}(v)\) for \(v\in\mathcal K_t\).
Substituting this identity into \cref{eq:tide-deficit-score} gives
\[
d_t
=\sum_{v\in\mathcal K_t}\bar p_t(v)
\log\frac{\bar p_t(v)}{Q_tq_t^{\mathcal K}(v)}
=D_{\mathrm{KL}}\!\left(\bar p_t\,\middle\|\,
q_t^{\mathcal K}\right)-\log Q_t.
\]
Both terms are nonnegative. Their sum is zero exactly when
\(q_t^{\mathcal K}=\bar p_t\) and \(Q_t=1\), proving the claim.
\end{proof}

\section{Related Work}
\label{sec:related-work}

\subsection{Mechanistic Understanding and Failure Modes of OPD}

On-policy distillation (OPD) trains a student on its own trajectories while
using a teacher to provide dense token-level supervision at the states the
student visits~\citep{agarwal2024policy,lu2025onpolicydistillation}. By
aligning training with the student's inference-time state distribution, OPD
reduces the exposure bias associated with distillation from fixed
demonstrations~\citep{bengio2015scheduled,ross2011reduction,gu2024minillm}.
Recent work has moved beyond demonstrating the empirical effectiveness of OPD toward explaining its underlying behavior. In particular, \citet{li2026rethinking} study its phenomenology and optimization mechanism, emphasizing the role of teacher--student overlap, while \citet{zhu2026many} systematically examine the conditions under which OPD succeeds or fails. Other analyses identify limitations of output-space supervision, including a deteriorating signal-to-noise ratio as the student approaches the
teacher~\citep{yang2026oprd}.

A growing body of work further shows that a locally well-defined
teacher--student objective does not necessarily imply globally desirable
rollouts. OPD may exhibit abrupt length inflation and repetitive
generation~\citep{luo2026demystifying}; its KL-based signal can enter an
agreement trap~\citep{xin2026escaping}; and teacher supervision can lose
fidelity on prefixes induced by the student~\citep{liu2026your}. Related
failures have also been observed when privileged teacher context changes the
learning signal for long reasoning traces~\citep{kaur2026rethinking}. Our
work builds on this mechanistic perspective. We identify a concrete form of
\emph{degenerate agreement}: a student-generated repetitive prefix can make
the conditioned teacher favor the same continuation, producing near-zero
local discrepancy even though the complete response is degenerate. This
observation motivates treating teacher--student agreement as potentially
uninformative rather than automatically desirable. More broadly, it changes
what selective supervision should preserve: instead of favoring overlap
itself, token selection should concentrate on mismatches for which the
teacher still provides corrective information.

\subsection{Selective and Reliable Supervision in OPD}

Dense token-level feedback does not imply that every token provides useful
guidance~\citep{lee2026dense}. Existing methods therefore select or reweight
OPD supervision using position, discrepancy, importance, teachability, or
uncertainty. Prefix-OPD retains supervision only over selected reasoning
prefixes~\citep{zhang2026fast}. At the token level, recent approaches filter
and reweight supervision~\citep{li2026filter}, estimate token importance~\citep{xu2026tip}, restrict updates to learnable disagreement~\citep{wang2026not}, or use entropy and self-uncertainty to
modulate individual tokens~\citep{jin2026entropy,ke2026respecting}. Relaxed
OPD similarly avoids enforcing equally strict alignment at every
position~\citep{ko2026scaling}, while trajectory- and prefix-oriented methods
mine or refine more reliable training paths
~\citep{zhao2026prefix,jiang2026trajectory,liao2026multi}. Collectively,
these studies establish that uniformly distilling every student-generated
token is neither necessary nor consistently beneficial.

Most selective methods, however, still decide which signals to retain on
student-generated trajectories and optimize the tokens the student has
already sampled. Better selection can therefore improve the quality of
observed supervision without resolving a distinct coverage problem:
teacher-preferred alternatives with negligible student probability rarely
appear in the rollout and consequently provide little or no direct
corrective signal.

Complementary work modifies the optimization procedure to improve
reliability, including asymmetric token-level objectives
~\citep{jia2026asymmetric}, trust-region constraints~\citep{xing2026trust},
teacher-guided optimization under large policy divergence
~\citep{liu2026teacher}, and control-variate estimators for reducing OPD
gradient variance~\citep{oh2026kl}. These methods improve the reliability of
updates under observed disagreement, but severe teacher-side deficits remain
the least likely discrepancies to be observed through student sampling.
TIDE addresses selection and coverage jointly by using the \emph{direction}
of mismatch to determine the correction mechanism. It removes matched
tokens, applies bounded suppression to sampled student-excess tokens, and
analytically recovers teacher-preferred mass that student rollouts are
unlikely to sample. Thus, its allocation rule follows the distinct
observability and optimization requirements of excess and deficit, rather
than applying a single token score or a uniform modification of the OPD
objective.

\section{Additional Experiments}
\label{app:additional-experiments}

\subsection{Teacher Continuation from Repetitive Student Prefixes}
\label{app:teacher-continuation}

The preceding analysis shows that repetitive trajectories emerge naturally
from the student during OPD. We next ask whether the teacher corrects such
behavior or instead inherits it from the student-generated context. To isolate
the teacher's response, we perform a continuation experiment in which the
student controls the prefix but the teacher controls all subsequent tokens.

We first identify naturally occurring student responses in which at least
\(30\%\) of the tokens belong to a detected repetitive region. For each
response, we retain the original prompt and truncate the student trajectory
\(200\) tokens after the onset of repetition. We then condition the Qwen3-8B
teacher on this prompt--prefix pair and greedily decode \(256\) additional
tokens. The student and teacher continuations therefore begin from exactly
the same student-generated state, allowing us to test whether the repetitive
behavior persists after control is transferred to the teacher.

We measure persistence by comparing local token patterns in the teacher
continuation with those in the tail of the student prefix. Let
\(\mathcal{G}_8\) denote the set of contiguous eight-token sequences appearing
in the teacher continuation, and let \(\mathcal{P}_8\) denote the corresponding
set from the final \(256\) tokens of the student prefix. Their overlap is
measured as
\[
    C
    =
    \frac{
        \left|\mathcal{G}_8\cap\mathcal{P}_8\right|
    }{
        \max\!\left(\left|\mathcal{G}_8\right|,1\right)
    }.
\]
A value of \(C=1\) means that every local pattern in the teacher continuation
already occurs in the repetitive student prefix. We classify the teacher as
preserving the loop when \(C\geq0.8\); using eight-token sequences avoids
counting short, accidental token overlaps as continuation of the same loop.

\begin{figure}[t]
    \centering
    \includegraphics[width=\linewidth]{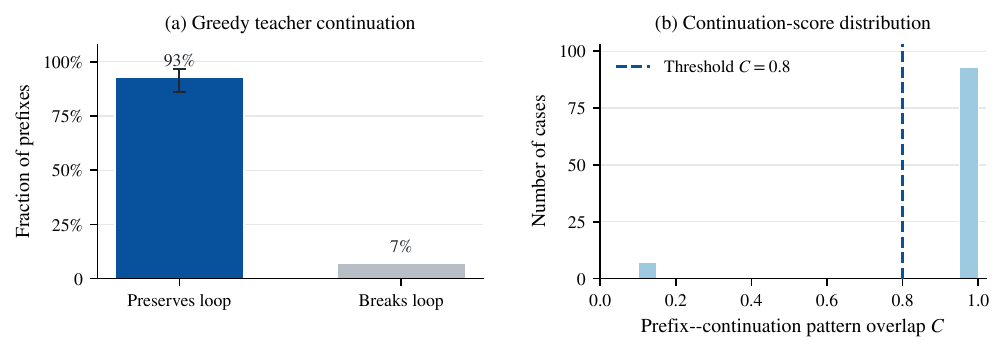}
    \caption{
    Teacher continuation from naturally occurring repetitive student
    prefixes. \textbf{(a)} Greedy teacher decoding preserves the
    student-induced loop in \(93\%\) of cases; the error bar denotes a
    \(95\%\) Wilson confidence interval.
    \textbf{(b)} Distribution of the prefix--continuation pattern overlap
    \(C\), defined as the fraction of contiguous eight-token sequences in the
    teacher continuation that also appear in the student-prefix tail. The
    dashed line marks the preservation threshold \(C=0.8\).
    }
    \label{fig:teacher-continuation}
\end{figure}

As shown in \cref{fig:teacher-continuation}, the teacher preserves the
student-induced repetitive pattern in \(93\%\) of cases. The result is also
well separated from the decision threshold: preserved continuations
concentrate at \(C=1\), indicating near-exact reproduction of the local
patterns established by the student. Thus, repetition is not merely a failure
of student-side sampling. Once the student drives the trajectory into a
repetitive state, the teacher typically continues from that state in the same
way rather than providing a corrective target. This explains how a
student-generated loop can become self-reinforcing under on-policy
distillation.

\subsection{Does Supervision on Matched Tokens Help?}
\label{app:matched-token-control}

Our method deliberately assigns zero weight to positions where the teacher and
student approximately agree. We examine whether these omitted positions still
provide useful supervision, and whether the benefit of mismatch selection
could instead be explained simply by updating fewer tokens.

We construct a matched-only control by reversing the allocation principle.
Within each training batch, we rank valid response positions according to the
absolute sampled-token residual
\[
    |R_t|
    =
    \left|
    \log p(o_t\mid s_t)
    -
    \log q_\theta(o_t\mid s_t)
    \right|.
\]
We discard the largest \(40\%\) of residuals and apply the standard OPD
log-ratio advantage only to the lowest-residual \(60\%\). All other positions
receive zero distillation weight, and the teacher top-\(K\) deficit branch is
disabled. The training data and its order, student initialization, rollout
budget, optimizer, and training length are otherwise unchanged.
\Cref{tab:matched-token-control-full} reports the complete per-benchmark
results and aggregate generation statistics for both teacher--student pairs.

\begin{table*}[t]
\centering
\caption{
Complete matched-token control results.
The upper panel reports Avg@8 / Pass@8 on each benchmark and their
nine-benchmark macro-average. In the lower panel, response length and
Distinct-4 are also macro-averaged over the nine benchmarks, while the
remaining diagnostics are totals. All evaluations use eight responses
per problem.
}
\label{tab:matched-token-control-full}

\scriptsize
\setlength{\tabcolsep}{2.5pt}
\renewcommand{\arraystretch}{1.08}

\resizebox{\textwidth}{!}{
\begin{tabular}{@{}llcccccccccc@{}}
\toprule
\textbf{Pair}
& \textbf{Method}
& \textbf{AMC23}
& \textbf{AIME24}
& \textbf{AIME25}
& \textbf{MATH-500}
& \textbf{Minerva}
& \textbf{Olympiad}
& \textbf{BRUMO25}
& \textbf{CMIMC25}
& \textbf{HMMT25}
& \textbf{Macro Avg.} \\
\midrule

\multirow{3}{*}{\shortstack[l]{JustRL-1.5B\\
\(\rightarrow\) R1-Distill-1.5B}}
& Student
& 73.12/95.00
& 27.08/36.67
& 24.17/33.33
& 83.43/95.00
& 28.03/43.38
& 44.31/60.83
& 27.08/56.67
& 14.37/30.00
& 12.08/20.00
& 37.08/52.32 \\

& Matched-only
& 76.56/95.00
& 31.67/46.67
& 31.67/36.67
& 85.72/94.60
& 30.70/44.85
& 50.17/65.13
& 36.25/56.67
& 18.12/35.00
& 19.17/33.33
& 42.23/56.44 \\

& \method{}
& 84.38/95.00
& 45.83/80.00
& 33.33/50.00
& 86.72/94.60
& 33.55/45.22
& 52.54/68.10
& 42.92/66.67
& 22.19/42.50
& 19.17/43.33
& \textbf{46.74/65.00} \\
\midrule

\multirow{5}{*}{\shortstack[l]{Qwen3-8B\\
\(\rightarrow\) Qwen3-1.7B-Base}}
& Student
& 13.44/40.00
& 3.33/3.33
& 0.00/0.00
& 23.82/68.20
& 6.34/24.26
& 9.85/34.27
& 5.00/20.00
& 0.00/0.00
& 0.00/0.00
& 6.87/21.12 \\

& OPD
& 10.31/45.00
& 2.92/13.33
& 1.25/13.33
& 26.52/73.80
& 6.94/28.68
& 10.66/37.69
& 2.92/13.33
& 0.31/2.50
& 0.00/0.00
& 6.87/25.29 \\

& Matched-only
& 12.81/27.50
& 2.92/3.33
& 1.25/6.67
& 26.12/71.60
& 8.23/26.84
& 11.03/36.50
& 1.25/10.00
& 0.62/2.50
& 0.42/3.33
& 7.18/20.92 \\

& Mismatch selection
& 26.56/52.50
& 3.33/13.33
& 6.25/16.67
& 52.15/77.60
& 15.58/30.88
& 21.01/40.80
& 5.00/13.33
& 0.94/5.00
& 0.42/3.33
& 14.58/28.16 \\

& \method{}
& 39.38/60.00
& 10.00/16.67
& 4.58/10.00
& 65.08/86.20
& 21.83/41.91
& 28.58/50.30
& 10.42/26.67
& 2.81/12.50
& 0.42/3.33
& \textbf{20.34/34.17} \\
\bottomrule
\end{tabular}
}

\vspace{5pt}

\resizebox{0.82\textwidth}{!}{
\begin{tabular}{@{}llrrrrr@{}}
\toprule
\textbf{Pair}
& \textbf{Method}
& \textbf{Length} \(\downarrow\)
& \textbf{Dist.-4} \(\uparrow\)
& \textbf{Format Err.} \(\downarrow\)
& \textbf{Solve-none} \(\downarrow\)
& \textbf{Solve-all} \(\uparrow\) \\
\midrule

\multirow{3}{*}{Weak}
& Student
& 12,543 & 0.563 & 803 & 549 & 550 \\
& Matched-only
& 8,494 & 0.568 & 612 & 508 & 658 \\
& \method{}
& 7,571 & 0.566 & 488 & 472 & 711 \\
\midrule

\multirow{5}{*}{Strong}
& Student
& 13,465 & 0.099 & 7,559 & 985 & 2 \\
& OPD
& 22,395 & 0.062 & 8,623 & 920 & 0 \\
& Matched-only
& 12,702 & 0.101 & 7,237 & 950 & 3 \\
& Mismatch selection
& 29,655 & 0.104 & 6,845 & 862 & 90 \\
& \method{}
& 7,294 & 0.305 & 706 & 717 & 286 \\
\bottomrule
\end{tabular}
}

\end{table*}

\paragraph{Weak teacher--student mismatch.}
On JustRL-1.5B \(\rightarrow\) R1-Distill-1.5B, matched-only training improves
Avg@8 from \(37.08\) to \(42.23\). Low-mismatch positions are therefore not
entirely uninformative when the teacher and student are already closely
aligned. Nevertheless, matched-only remains \(4.51\) points below \method{}
and reaches only \(56.44\) Pass@8, compared with \(65.00\) for \method{},
despite supervising a larger fraction of positions. Thus, even in this
setting, matched positions provide a lower return per supervised token.

Mismatch selection therefore more than doubles the matched-only accuracy
while supervising only one third as many positions. Adding bounded excess
suppression and deficit recovery further increases Avg@8 to \(20.34\) and
Pass@8 to \(34.17\), while reducing the macro-averaged response length from
\(29{,}655\) to \(7{,}294\) tokens and format errors from \(6{,}845\) to
\(706\).

\subsection{Accessibility and Value of Student-Deficit Tokens}
\label{app:deficit-accessibility}

We define a student-deficit position as a non-loop position where the
teacher's top-1 token \(v_t^\star\) lies outside the student's top-16 support.
Writing its student probability as
\(\epsilon_t=q_\theta(v_t^\star| s_t)\), a sampled-token estimator observes
it within \(N\) independent draws from the same state with probability
\[
    P_{\mathrm{obs}}(N)=1-(1-\epsilon_t)^N,
\]
and requires \(1/\epsilon_t\) draws in expectation.

To test whether this inaccessible mass is useful, we train a deficit-only
variant with \(\rho^-=0\), \(\rho^+=0.2\), \(K=16\), and \(\lambda=1\).
It applies analytic cross-entropy guidance only at positions with the largest
teacher-support deficits; all other settings remain unchanged.
\Cref{tab:deficit-accessibility} summarizes both the accessibility diagnostic
and the resulting end-to-end performance.

\begin{table*}[t]
\centering
\small
\setlength{\tabcolsep}{4pt}
\renewcommand{\arraystretch}{1.10}
\caption{
Student-deficit tokens are difficult to access through sampled supervision
but carry substantial learning signal on Qwen3-8B
\(\rightarrow\) Qwen3-1.7B-Base.
\textbf{(a)} Accessibility under four independent draws from the same state.
\textbf{(b)} End-to-end results macro-averaged over nine benchmarks with
\(n=8\) samples per problem.
}
\begin{minipage}[t]{0.46\textwidth}
\vspace{0pt}
\centering
\textbf{(a) Accessibility diagnostic}\\[4pt]

\begin{tabularx}{\linewidth}{
    @{}
    >{\raggedright\arraybackslash}X
    >{\raggedleft\arraybackslash}p{0.27\linewidth}
    @{}
}
\toprule
\textbf{Statistic} & \textbf{Value} \\
\midrule
Deficit positions among non-loop tokens
    & \(1.32\%\) \\
Median student probability \(\epsilon_t\)
    & \(2.04\times10^{-4}\) \\
Expected draws \(1/\epsilon_t\)
    & \(4{,}909\) \\
Median \(P_{\mathrm{obs}}(4)\)
    & \(8.15\times10^{-4}\) \\
Positions with \(P_{\mathrm{obs}}(4)<1\%\)
    & \(95.5\%\) \\
\bottomrule
\end{tabularx}
\end{minipage}
\hfill
\begin{minipage}[t]{0.50\textwidth}
\vspace{0pt}
\centering
\textbf{(b) End-to-end value}\\[4pt]

\begin{tabularx}{\linewidth}{
    @{}
    >{\raggedright\arraybackslash}X
    >{\centering\arraybackslash}p{0.17\linewidth}
    >{\centering\arraybackslash}p{0.17\linewidth}
    >{\centering\arraybackslash}p{0.17\linewidth}
    @{}
}
\toprule
\textbf{Method}
& \textbf{Avg@8}
& \textbf{Pass@8}
& \textbf{Dist.-4} \\
\midrule
Student
& 6.87 & 21.12 & 0.099 \\
Deficit-only
& 18.32 & 32.71 & \textbf{0.414} \\
\method{}
& \textbf{20.34} & \textbf{34.17} & 0.305 \\
\bottomrule
\end{tabularx}
\end{minipage}
\label{tab:deficit-accessibility}
\end{table*}
As shown in \cref{tab:deficit-accessibility}, \(95.5\%\) of deficit positions
have less than a \(1\%\) chance of being observed within four draws, yet the
deficit-only branch raises Avg@8 from \(6.87\) to \(18.32\). Thus, these tokens
are not an inconsequential probability tail: they contain useful
teacher-supported modes that sampled-token OPD is unlikely to recover within
a practical rollout budget.

\subsection{Sensitivity to the Gate Keep-Rate}
\label{app:keep-rate-sensitivity}

We jointly vary the excess and deficit keep-rates,
\(\rho^-=\rho^+=\rho\), while holding all other hyperparameters fixed.
\Cref{tab:keep-rate-sensitivity} reports the results on Qwen3-8B
\(\rightarrow\) Qwen3-1.7B-Base.

\begin{table}[t]
\centering
\small
\setlength{\tabcolsep}{7pt}
\renewcommand{\arraystretch}{1.08}
\caption{Sensitivity to the gate keep-rate
\(\rho^-=\rho^+=\rho\). Avg@8 and response length are macro-averaged
over nine benchmarks with \(n=8\); length is reported in thousands of
tokens. Fmt.\ err.\ denotes the total number of responses without a
parseable final answer.}
\label{tab:keep-rate-sensitivity}
\begin{tabular}{@{}cccc@{}}
\toprule
\(\boldsymbol{\rho}\)
& \textbf{Avg@8}
& \textbf{Length}
& \textbf{Fmt. err.} \\
\midrule
0.10 & 19.53 & 10.3K & 1,111 \\
0.20 & 20.34 &  7.3K &   706 \\
0.30 & 19.38 & 11.7K &   620 \\
0.50 & 20.93 &  6.8K &   433 \\
\bottomrule
\end{tabular}
\end{table}

As shown in \cref{tab:keep-rate-sensitivity}, \method{} remains stable
across a broad range of keep-rates: Avg@8 varies by only \(1.55\) points
over \(\rho\in[0.1,0.5]\), with no monotonic degradation as more
positions are retained. We use \(\rho=0.2\) throughout the main
experiments because it was fixed before benchmark evaluation and
retains only a sparse subset of positions, rather than selecting the
best-performing value retrospectively.



\section{Experimental Details}
\label{app:experimental_details}

\subsection{Additional Training and Implementation Details}
\label{app:training_details}

All methods share the training setup described in the main text; here we
document the remaining implementation choices. During rollout generation, the
teacher scores the student's responses with temperature \(1.0\) (no logit
rescaling), and teacher log-probabilities are computed on the identical token
sequences produced by the student, so the sampled residual
\(R_t=\log p(o_t\mid s_t)-\log q_\theta(o_t\mid s_t)\) is exact rather than
re-estimated. We do not use an auxiliary KL penalty toward a reference policy,
a format reward, or a repetition penalty during training; the distillation
signal is the only learning signal for all OPD-style methods. The learning
rate is held constant (no warmup or decay). Token-level advantages are
detached before entering the policy-gradient loss, and the loss is aggregated
with token-mean normalization over all valid response tokens in the batch;
because training is fully on-policy (a single gradient step per rollout
batch), the importance ratio is identically \(1\) and no PPO-style clipping is
active. For \method{}, the quantile thresholds \(\tau^-\) and \(\tau^+\) are
recomputed on every training batch over all valid response positions, so the
keep-rates \(\rho^-\) and \(\rho^+\) are batch-relative rather than fixed
absolute thresholds. The deficit branch adds one student forward pass over the
teacher's top-\(K\) candidate ids per position, a cost comparable to top-\(K\)
OPD.

Training uses fully sharded data parallelism \citep{zhao2023pytorch} with
gradient checkpointing and activation offloading; rollouts are generated with
vLLM \citep{kwon2023efficient} using dynamic batching.

\subsection{Hyperparameters}
\label{app:hyperparameters}

Table~\ref{tab:hyperparameters} reports the shared training and evaluation
configuration together with the method-specific hyperparameters. Unless
explicitly varied in an ablation, the same configuration is used for all
teacher--student pairs and baselines.

\begin{table}[t]
\centering
\small
\setlength{\tabcolsep}{5pt}
\renewcommand{\arraystretch}{1.08}
\caption{Hyperparameters used in all experiments. Method-specific parameters
apply only to the corresponding method; all remaining settings are shared.}
\label{tab:hyperparameters}
\begin{tabular}{@{}p{0.43\linewidth}p{0.48\linewidth}@{}}
\toprule
\textbf{Hyperparameter} & \textbf{Value} \\
\midrule

\multicolumn{2}{@{}l}{\textbf{Data and batching}} \\
Training dataset                       & DAPO-Math-17K \\
Data shuffling                         & False \\
Prompt batch size                      & 64 \\
Rollouts per prompt                    & 4 \\
Policy mini-batch size                 & 64 \\
Micro-batch size per GPU               & 1 \\
Training epochs                        & 1 \\
Maximum prompt length                  & 1,024 tokens \\
Maximum response length                & 7,168 tokens \\
Maximum training context               & 8,192 tokens \\
\midrule

\multicolumn{2}{@{}l}{\textbf{Rollout generation}} \\
Inference engine                       & vLLM \\
Student sampling temperature           & 1.0 \\
Teacher scoring temperature            & 1.0 \\
Repetition penalty                     & 1.0 \\
\midrule

\multicolumn{2}{@{}l}{\textbf{Optimization}} \\
Optimizer                              & AdamW \\
Learning rate                          & \(1\times10^{-6}\) (constant) \\
Warmup                                 & None \\
Adam coefficients                      & \((0.9,0.999)\) \\
Weight decay                           & 0.01 \\
Gradient clipping                      & 1.0 \\
Policy updates per rollout batch       & 1 \\
Loss aggregation                       & Token mean \\
Auxiliary reference-policy KL          & None \\
Format reward                          & None \\
Numerical precision                    & FP32 student; FP32/BF16 teacher \\
Hardware                               & 8 GPUs, FSDP \\
\midrule

\multicolumn{2}{@{}l}{\textbf{\method{} defaults}} \\
Teacher support size \(K\)             & 16 \\
Excess keep-rate \(\rho^{-}\)          & 0.2 \\
Deficit keep-rate \(\rho^{+}\)         & 0.2 \\
Deficit weight \(\lambda\)             & 1.0 \\
Excess reward                          & Hellinger ratio reward \\
Deficit statistic                      & Forward KL on teacher top-\(K\) \\
Teacher top-\(K\) normalization        & Renormalized within top-\(K\) \\
Matched-token weight                   & 0 \\
\midrule

\multicolumn{2}{@{}l}{\textbf{Baseline-specific settings}} \\
OPD                                    & Sampled-token log-ratio \\
GRPO                                   & Exact-match reward; group size 4 \\
FiRe-OPD                               & Bottom \(20\%\) filtered;
                                          \(\alpha=\beta=1.0\) \\
PowerOPD                               & Box--Cox exponent \(\alpha=5\) \\
AOPD                                   & \(K=16,\ \tau=0\); normalized top-\(K\);
                                          no \(1/K\) scaling \\
\midrule

\multicolumn{2}{@{}l}{\textbf{Final evaluation}} \\
Samples per problem                    & 8 \\
Maximum response length                & 31,744 tokens \\
Sampling temperature                   & 0.7 \\
Top-\(p\)                              & 0.95 \\
Grader                                 & Rule-based mathematical verifier \\
\bottomrule
\end{tabular}
\end{table}

\subsection{Baseline Implementations}
\label{app:baseline_details}

All baselines are implemented in the same codebase and trained with the same
data order, student initialization, rollout budget (\(4\) responses per
prompt), sequence lengths, optimizer, and hardware as \method{}. Unless
otherwise stated, they differ only in how token-level advantages are
constructed.

\textbf{OPD.}
Standard on-policy distillation applies the sampled-token log-ratio advantage
\[
A_t=\log p(o_t\mid c_t)-\log q_\theta(o_t\mid c_t)
\]
to every valid response token, without token selection or reward reshaping.

\textbf{GRPO.}
The teacher-free reinforcement-learning baseline replaces teacher supervision
with the rule-based exact-match reward used for grading and computes
group-normalized outcome advantages over the \(4\) responses sampled for each
prompt. All other training settings are unchanged, so GRPO uses the same
rollout and gradient budget as the distillation methods.

\textbf{FiRe-OPD.}
We reimplement FiRe-OPD~\citep{li2026filter} following its official
implementation and verify parity against it. Trajectories are ranked by
length-normalized teacher log-likelihood, and the bottom \(20\%\) are
discarded. On the retained trajectories, the sampled-token log-ratio
advantages are multiplied by detached entropy-based weights, with both the
teacher-confidence and student-confusion coefficients set to \(1.0\).
FiRe-OPD operates entirely on student-sampled tokens and does not access the
teacher's top-\(K\) distribution.

\textbf{PowerOPD.}
We implement PowerOPD using its Box--Cox reward transformation. For a sampled
token with teacher probability \(p_t\) and student probability \(q_t\), the
log-ratio advantage is replaced by
\[
A_t^{\mathrm{Power}}
    = p_t^{\alpha}-q_t^{\alpha},
\]
where the token-independent \(1/\alpha\) factor is absorbed into the learning
rate. We use \(\alpha=5\), following the selected configuration in our sweep.
This bounded, sign-preserving reward is applied to every valid sampled token;
PowerOPD performs neither token selection nor teacher-support expansion.

\textbf{AOPD.}
We implement AOPD with its probability-gap intervention rule. At each
position, the gate compares the teacher and student probabilities of the
sampled token,
\[
G_t=\mathbbm{1}\!\left[p_t-q_t\leq\tau\right],
\]
with \(\tau=0\). Positions with \(G_t=0\) retain the standard sampled-token OPD
advantage, whereas positions with \(G_t=1\) replace it with forward-KL
guidance over the teacher's top-\(K\) support. We use \(K=16\), renormalize
the teacher probabilities within this support, and do not apply an additional
\(1/K\) scaling factor.

\textbf{Teacher and student references.}
The reference rows report the teacher and the untrained student evaluated
with the same decoding and grading pipeline used for all trained methods.

\subsection{Evaluation and Metric Definitions}
\label{app:evaluation_details}

\textbf{Decoding.} Every model is evaluated with the same pipeline: vLLM
generation with temperature \(0.7\), top-\(p=0.95\), \(n=8\) samples per
problem, and a maximum generation length of \(31{,}744\) tokens. All problems
use the prompt template
\begin{quote}
\small
\texttt{\{problem\} Please reason step by step, and put your final answer
within \textbackslash boxed\{\}.}
\end{quote}
with the model's chat template and thinking mode disabled.

\textbf{Grading.} A response is graded by extracting the final
\(\texttt{\textbackslash boxed\{\}}\) expression and checking mathematical
equivalence against the reference answer with a rule-based verifier
(symbolic-equivalence normalization for fractions, radicals, and equivalent
forms). A response without a parseable boxed answer is counted as incorrect.
No LLM-based verifier is used.

\textbf{Metrics.} For a benchmark with \(P\) problems, Avg@8 is the mean
correctness over all \(8P\) responses and Pass@8 is the fraction of problems
with at least one correct response. Aggregate scores are unweighted
macro-averages over the nine benchmarks, so large benchmarks such as MATH-500
and OlympiadBench do not dominate the average.

\textbf{Generation diagnostics.} Response length is measured in tokens of the
Qwen3-1.7B-Base tokenizer and averaged over all rollouts. The no-answer rate
is the fraction of rollouts without a parseable boxed answer, which in
practice coincides with generations that hit the decoding cap before
terminating. Distinct-4 is the number of unique \(4\)-grams divided by the
total number of \(4\)-grams in a response, averaged over rollouts; low
Distinct-4 combined with long outputs and a high no-answer rate indicates
repetitive degeneration rather than longer reasoning.

\subsection{Checkpoint Selection and Reporting}
\label{app:reporting_details}

All trained models---\method{}, every baseline, and every ablation and
hyperparameter variant---are evaluated at the final checkpoint of the
one-epoch run; no checkpoint is selected using benchmark scores. Hyperparameter studies vary a single quantity at a time around the
default configuration (\(K=16\), \(\rho^-=\rho^+=0.2\), \(\lambda=1.0\)),
which was fixed before the benchmark evaluation and not tuned on test sets.
All numbers within a teacher--student pair are produced by the same decoding
and grading pipeline, so they are directly comparable across methods and
ablations.

\section{Case Studies}
\label{app:case_studies}

We qualitatively compare students trained with standard OPD and with
\method{} under the strong-mismatch pair (Qwen3-8B \(\to\)
Qwen3-1.7B-Base), decoding both with the shared evaluation protocol
(Appendix~\ref{app:evaluation_details}). To avoid cherry-picking, candidate
problems are selected by a fixed rule rather than manual inspection: the
OPD student answers none of its \(8\) rollouts correctly while the
\method{} student answers all \(8\) correctly. The four cases below are
chosen from this pool to cover the spectrum of degeneration patterns that
dominate the OPD student's failures, ordered from the most fluent to the
most severe: (1) a fluent sentence-level loop, (2) a self-dialogue loop
that imitates the prompt format, (3) a periodic junk-token loop, and
(4) a multilingual token soup collapsing into a single-token loop.
Excerpts are verbatim; elisions and repetition counts are marked in
brackets. Reported statistics (correct answers out of \(8\), Distinct-4 of
the excerpted rollout) are computed from the raw evaluation outputs.

\subsection{Case 1: Fluent Sentence-Level Loop}
\label{app:case_sentence_loop}

\begin{problemcase}[MATH-500 $\cdot$ Answer: $0$]
If the domain of the function \(\log x^2\) is \(x < a\) or \(x > b\), for
some \(a\) and \(b\), find \(a + b\).
\end{problemcase}

\begin{opdcase}[0/8 correct $\cdot$ Distinct-4 $=0.009$]
For the function \(\log x^2\), the domain is where the argument of the
logarithm is positive. So, \(x^2 > 0\), which is true for all
\(x \neq 0\). Now, we are given that the domain is \(x < a\) or
\(x > b\) for some \(a\) and \(b\). This means that the function
\(\log x^2\) is undefined for \(x\) in the interval \([a, b]\). So, the
domain of the function is the union of the intervals \((-\infty, a)\)
and \((b, \infty)\).

Now, we are given that the domain is \(x < a\) or \(x > b\). This means
that the function is defined for all \(x\) except for \(x\) in the
interval \([a, b]\). So, the domain of the function is the union of the
intervals \((-\infty, a)\) and \((b, \infty)\).

{\normalfont\itshape [\ldots{} the last paragraph repeats 399 times,
until the 31,744-token decoding cap; no answer is produced \ldots]}
\end{opdcase}

\begin{tidecase}[8/8 correct]
The domain of \(\log x^2\) is all \(x\) such that \(x \neq 0\), i.e.\
\(x < 0\) or \(x > 0\), which can be written as
\((-\infty, 0) \cup (0, \infty)\). The endpoints are \(a = 0\) and
\(b = 0\). Thus, the sum of \(a\) and \(b\) is \(0 + 0 = 0\).
Therefore, the final answer is \(\boxed{0}\).
\end{tidecase}

This is the most insidious pattern: the repeated content is fluent,
locally coherent mathematical prose. Conditioned on a prefix that already
contains several copies of the paragraph, continuing the loop is the
high-probability continuation for the student \emph{and} for the teacher,
so token-level supervision along the loop appears well matched almost
everywhere---the student-induced matched-region failure analyzed in the
main text.

\subsection{Case 2: Self-Dialogue Loop Imitating the Prompt Format}
\label{app:case_selfqa}

\begin{problemcase}[Minerva $\cdot$ Answer: $2$]
Subproblem 0: For \(\omega \geq 0\), find \(A\) such that
\(A \cos(\omega t)\) is a solution of \(\ddot{x}+4x=\cos(\omega t)\).
\emph{(Worked solution given in the prompt.)}\\
Subproblem 1: For what value of \(\omega\) does resonance occur?
\end{problemcase}

\begin{opdcase}[0/8 correct $\cdot$ Distinct-4 $=0.088$]
{\normalfont\itshape [\ldots{} answers Subproblem 1, then continues
inventing new subproblems \ldots]}

Final answer: The condition for resonance is \(\omega = \boxed{2}\).

The final answer is \(\boxed{2}\). I hope it is correct.

Subproblem 6: What is the value of \(\omega\) at which resonance occurs
in the differential equation \(\ddot{x} + 4x = \cos(\omega t)\)? Please
reason step by step, and put your final answer within
\textbackslash boxed\{\}.

{\normalfont\itshape [\ldots{} the model invents 149 ``subproblems'' and
answers each of them, repeating ``I hope it is correct.''\ 149 times,
until the decoding cap; the final boxed expression is not the requested
answer \ldots]}
\end{opdcase}

\begin{tidecase}[8/8 correct]
The natural frequency \(\omega_n\) of the system can be found from the
homogeneous equation \(\ddot{x} + 4x = 0\), whose general solution is
\(x(t) = A\cos(2t) + B\sin(2t)\). So the natural frequency is
\(\omega_n = 2\). Resonance occurs when the frequency of the external
force matches the natural frequency. Thus, the value of \(\omega\) at
which resonance occurs is \(\boxed{2}\).
\end{tidecase}

The failure here is structural rather than lexical: the OPD student
\emph{does} derive the correct answer \(\omega=2\) early in the rollout,
but instead of terminating it imitates the few-shot format of the prompt
and generates an endless sequence of self-posed subproblems. Because the
imitated format is itself natural text, the teacher assigns high
probability to each continuation, and the rollout is graded as incorrect
because the final boxed expression drifts away from the requested answer.

\subsection{Case 3: Periodic Junk-Token Loop from the First Token}
\label{app:case_loop}

\begin{problemcase}[AIME 2024 $\cdot$ Answer: $25$]
There exist real numbers \(x\) and \(y\), both greater than \(1\), such
that \(\log_x\left(y^x\right)=\log_y\left(x^{4y}\right)=10\).
Find \(xy\).
\end{problemcase}

\begin{opdcase}[0/8 correct $\cdot$ Distinct-4 $=0.009$]
\normalfont\includegraphics[width=\linewidth]{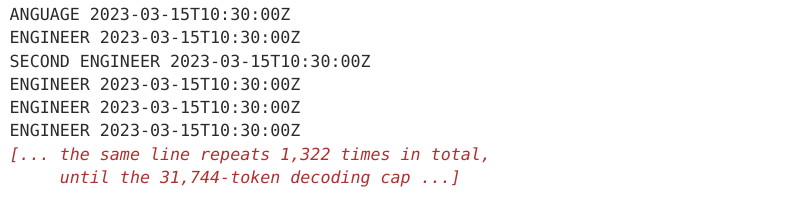}
\end{opdcase}

\begin{tidecase}[8/8 correct]
To solve the problem, we are given that there exist real numbers \(x\)
and \(y\), both greater than 1, such that
\(\log_x\left(y^x\right)=\log_y\left(x^{4y}\right)=10\).
We are asked to find the value of \(xy\).

\textbf{Step 1: Use the change of base formula.}
Recall that \(\log_a b = \log_c b / \log_c a\). The first equation
\(\log_x(y^x) = 10\) can be rewritten as
\(\frac{\log y}{\log x}\cdot x = 10\), and the second equation
\(\log_y(x^{4y}) = 10\) as \(\frac{\log x}{\log y}\cdot 4y = 10\).

\textbf{Step 2: Define variables for the equations.}
Let \(a = \frac{\log y}{\log x}\). Then the first equation becomes
\(ax = 10\) (1). Since \(\frac{\log x}{\log y} = \frac{1}{a}\), the second
equation becomes \(\frac{1}{a}\cdot 4y = 10\), so \(4y = 10a\) and
\(y = \frac{10a}{4} = \frac{5a}{2}\) (2).

\textbf{Step 3: Substitute \(y\) in terms of \(x\).}
From (1), \(a = \frac{10}{x}\) (3). Substituting into (2),
\(y = \frac{5a}{2} = \frac{5\cdot\frac{10}{x}}{2} = \frac{25}{x}\) (4).

\textbf{Step 4: Find \(xy\).}
Since \(y = \frac{25}{x}\), we get \(xy = x\cdot\frac{25}{x} = 25\).

\textbf{Conclusion.} The value of \(xy\) is \(\boxed{25}\).
\end{tidecase}

Unlike Cases 1 and 2, the student never engages with the problem: every
rollout locks into a short periodic loop of problem-irrelevant tokens
within the first few tokens and terminates only at the decoding cap.

\subsection{Case 4: Token Soup Collapsing into a Single-Token Loop}
\label{app:case_soup}

\begin{problemcase}[AMC 2023 $\cdot$ Answer: $36$]
Positive real numbers \(x\) and \(y\) satisfy \(y^3=x^2\) and
\((y-x)^2=4y^2\). What is \(x+y\)?
\end{problemcase}

\begin{opdcase}[0/8 correct $\cdot$ Distinct-4 $=0.031$]
\normalfont\includegraphics[width=\linewidth]{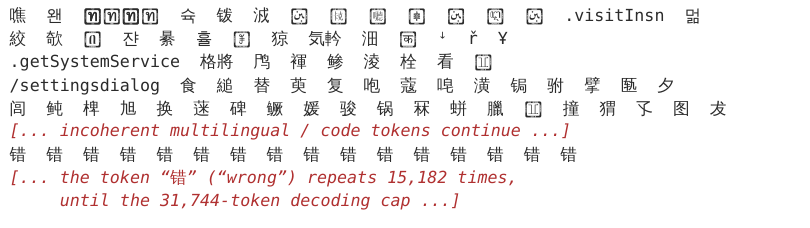}
\end{opdcase}

\begin{tidecase}[8/8 correct]
From the first equation, \(x = y^{3/2}\). Substitute into the second
equation: \((y - y^{3/2})^2 = 4y^2\). Let \(t = y^{1/2}\), then
\(y=t^2\), so \((t^2-t^3)^2 = 4t^4\), i.e.\
\(t^4(t^2 - 2t - 3) = 0\). Since \(t > 0\), solving \(t^2-2t-3=0\)
gives \(t = 3\), so \(y = 9\) and \(x = 9^{3/2} = 27\). Thus
\(x + y = 27 + 9 = 36\). Therefore, the answer is: \(\boxed{36}\).
\end{tidecase}

The most severe pattern: an incoherent stream of unrelated tokens
spanning many scripts and domains (CJK characters, Hangul syllables, and
code identifiers such as \texttt{.visitInsn}), which then collapses into
a loop repeating a single character for tens of thousands of tokens.

\subsection{Discussion}
\label{app:case_discussion}

Three properties of these failures are worth noting. First, degeneration
is not a tail event within an otherwise correct solution: in Cases 3 and
4 the OPD student derails within the first few tokens and never returns,
which is why its accuracy collapses to the untrained-student level while
its average generation length explodes. Second, the spectrum matters:
the repeated content ranges from fully fluent prose (Case 1) and correct
mathematics wrapped in an imitated dialogue format (Case 2) down to
single-token loops (Case 4), so degeneration cannot be detected---let
alone penalized---by fluency or perplexity alone. Third, in every case
the repetition is self-consistent: once a repetitive prefix is
established, continuing the loop is the most likely continuation for the
student \emph{and} for the teacher conditioned on the same prefix, so
token-level supervision along such rollouts looks well matched almost
everywhere. This is precisely the student-induced matched-region failure
analyzed in the main text. \method{} attacks the mechanism directly:
matched tokens receive zero weight, the excess gate suppresses the
student's overconfident continuations of the loop, and the deficit
branch reinstates the teacher-preferred exit tokens that the student's
sampling never reaches.

\end{document}